\documentclass[twoside,11pt]{article}

\usepackage[dvipsnames]{xcolor}
\usepackage{color}
\usepackage{jmlr2e}
\usepackage{mathrsfs,amsmath,amssymb,bm}
\usepackage{graphics,graphicx,epsfig,subfigure}
\def\B{{\mathcal{B}}}

\jmlrheading{0}{0000}{0-0}{0/00}{00/00}{Jian Du, Shaodan Ma, Yik-Chung Wu, Soummya Kar, and Jos\'{e} M. F. Moura}

\firstpageno{1}

\begin{document}

\title{Convergence Analysis of Distributed Inference with Vector-Valued Gaussian Belief Propagation}

\author{\name Jian Du \email jiand@andrew.cmu.edu \\
       \addr Department of Electrical and Computer Engineering\\
       Carnegie Mellon University\\
       Pittsburgh, PA 15213, USA
       \AND
       \name Shaodan Ma \email shaodanma@umac.mo \\
       \addr Department of Electrical and Computer Engineering\\
       University of Macau\\
       Avenida da Universidade, Taipa, Macau
       \AND
       \name Yik-Chung Wu \email ycwu@eee.hku.hk \\
       \addr Department of Electrical and Electronic Engineering\\
       The University of Hong Kong\\
       Pokfulam Road, Hong Kong
       \AND
       \name Soummya Kar
        \email soummyak@andrew.cmu.edu\\
       \name Jos\'{e} M. F. Moura
        \email moura@andrew.cmu.edu \\
       \addr Department of Electrical and Computer Engineering\\
       Carnegie Mellon University\\
       Pittsburgh, PA 15213, USA
       }

\editor{}

\maketitle

\begin{abstract}
This paper considers inference  over  {distributed linear Gaussian models} using
factor graphs and   Gaussian belief propagation (BP).
The  distributed {inference} algorithm involves only local computation of the information matrix and of the mean vector, and message passing between neighbors.
Under broad conditions, it is shown
that the  message information matrix converges  to a unique positive definite limit matrix for arbitrary positive semidefinite initialization,
{and it approaches  an arbitrarily small neighborhood of this limit matrix at a doubly exponential rate.}
A necessary and sufficient convergence condition for the belief mean vector to converge to the optimal centralized estimator is provided
 under the assumption that
  the message information matrix is initialized as a
positive semidefinite matrix.
Further, it is shown that Gaussian BP always converges when the underlying factor graph is given by the union of a forest and a single loop.
The proposed convergence condition in the setup of distributed linear Gaussian models is shown to be strictly weaker than  other existing convergence conditions and requirements, including the Gaussian Markov random field based walk-summability condition, and  applicable to a large class of scenarios.
\end{abstract}


\section{Introduction}
Inference based on a set of measurements  from multiple agents on a distributed network is a central issue in many  problems.
While centralized  algorithms can be used   in small-scale networks, they face difficulties in  large-scale networks, imposing a heavy
communication burden when
all the data is to be transported to and processed at a  central processing unit.
Dealing with highly distributed data  has been recognized by the U.S.  National Research Council as one of the big challenges for processing big data \citep{MassiveData}.
Therefore, distributed {inference} techniques that
only involve local communication and computation are important for problems arising in  distributed networks.

In large-scale linear parameter {learning} with Gaussian measurements,
Gaussian Belief Propagation (BP)
 \citep{DiagnalDominant} provides an efficient distributed algorithm for computing the marginal means of the unknown parameters, and it has been adopted in a variety of topics including   image interpolation \citep{image-interp}, distributed power system state {inference} \citep{A1}, distributed beamforming \citep{A2}, distributed synchronization  \citep{JianClock}, fast solver for system of linear equations \citep{A4}, distributed rate control in ad-hoc networks \citep{A5},  factor analyzer network \citep{A6}, sparse Bayesian learning \citep{A7}, inter-cell interference mitigation \citep{A8}, and peer-to-peer rating in social networks \citep{A9}.

Although with great empirical success \citep{Murphy}, it is known that a major challenge that hinders   BP is the lack of theoretical guarantees of convergence in loopy networks \citep{chertkov2006loop, gomez2007truncating}.
Convergence of other forms of loopy BP are analyzed by \citet{Ihler05},  \citet{MooijKappen_UAI_05, MooijKappen_IEEETIT_07}, \citet{Martin13}, and \citet{Russell15}, but  their analyses are not directly
applicable to Gaussian BP.
Sufficient convergence conditions for Gaussian BP have been developed in \citet{DiagnalDominant,WalkSum1,minsum09,Suqinliang}  when the underlying Gaussian distribution is expressed in terms of pairwise connections between \textit{scalar} variables, i.e., it is a Markov random field (MRF).
However,
depending on how the underlying joint Gaussian distribution is factorized, Gaussian BP may exhibit different convergence properties as different factorizations (different Gaussian models)  lead to fundamentally different recursive update structures.
In this paper,  we study the convergence of Gaussian BP derived from the distributed linear Gaussian model.
The motivation is twofold. From the factorization
viewpoint,  by specifically employing a factorization based on the linear Gaussian model, we are able to bypass difficulties in existing
convergence analyses (\citep{WalkSum1} and references therein) based on Gaussian Markov random field factorization.
From the distributed inference viewpoint, the linear Gaussian model and associated message passing requirements for implementing the Gaussian BP
readily conform to the physical network topology
arising in large-scale networks such as in \citep{A1,A2,JianClock,A4,A5,A6,A7,A8,A9}, thus it is practically important.

{Recently, \cite{giscard2012walk, giscard2013evaluating, pathsum} present a path-sum method to compute the  information matrix  inverse of a joint Gaussian distribution.
Then, the marginal mean is obtained using the information matrix inverse.
The path-sum method converges for an arbitrary valid Gaussian model,  however, it is not clear how to adapt it to the distributed and parallel inference setup.
In contrast,
Gaussian BP is a parallel and fully distributed method that computes the marginal means   by  computing only the block diagonal elements of the information matrix inverse.
Though the block diagonal elements computed by Gaussian BP may not be correct, it is   shown that the belief mean  still converges to the  correct value  once Gaussian BP converges.
This explains the popularity of Gaussian BP in distributed inference applications, even though its convergence properties are not fully understood.}

{To fill this gap, this paper  studies the convergence of
	Gaussian BP for linear Gaussian models.
Specifically, for the first time, by establishing certain contractive properties of the distributed information matrix (inverse covariance matrix)  updates with respect to the Birkhoff metric, we show that, with
arbitrary positive
semidefinite  (p.s.d.)  initial message information matrix, the belief covariance for each local variable converges
 to a unique positive definite limit,  {and it approaches an arbitrarily small neighborhood of this limit matrix at a \textit{doubly exponential} rate.}
Consequently, the recursive equation for the message mean, which depends on the information matrix, can be reduced to a linear recursive equation.
Further, we derive
a necessary and sufficient convergence condition for this linear recursive equation
under the assumption that the initial message information matrix is p.s.d. Furthermore, we show that, when the structure of the factor graph is
  the union of a single loop and a forest, Gaussian BP always converges.
Finally, it is demonstrated that  the proposed convergence condition for the linear Gaussian model encompasses the walk-summable convergence condition for  Gaussian MRFs \citep{WalkSum1}.}

Note that there exist other distributed estimation frameworks, e.g.,
consensus$+$inn-ovations \citep{Kar-SPS, Kar-SIAM} and
diffusion algorithms~\citep{Sayed10Diffusion} that enable distributed estimation of parameters and processes in multi-agent networked environments.
The consensus$+$innovation algorithms converge in mean square sense to the centralized optimal solution
under the assumption of global observability
of the (aggregate) sensing model and connectivity (on the average) of the inter-agent communication network.
In particular, these algorithms allow the communication or message exchange network to be different from the physical coupling network of the field being {estimated} where either networks can be arbitrarily connected with cycles.
The results in \citet{Kar-SPS, Kar-SIAM} imply that the unknown field or parameter   can be reconstructed completely at each agent in the network.
For large-scale networks with high dimensional unknown variable, it may be impractical though to {estimate}  all the unknowns at every agent.
Reference \citep[section 3.4]{Kar-thesis} develops approaches to address this problem, where under appropriate conditions, each agent can {estimate} only a subset of the unknown parameter variables.
This paper studies a different distributed inference problem where each agent {learns} only its own unknown random variables;
this leads to lower dimensional data exchanges between neighbors.

The rest of this paper is organized as follows.
Section~\ref{hybrid} presents
the system model for distributed inference.
Section~\ref{vmpgamma} derives the
vector-valued distributed inference algorithm based on Gaussian BP.
Section~\ref{analysis} establishes convergence  conditions, and Section~\ref{relationship} discloses the relationship between the derived results and existing convergence conditions of Gaussian BP.
Finally,  Section~\ref{conclusion}  presents our conclusions.

\textit{Notation}: Boldface uppercase and lowercase letters represent  matrices and vectors, respectively.
For a matrix $\textbf {A}$,
$\textbf {A}^{-1}$ and $\textbf {A}^T$ denote  its inverse (if it exists) and transpose, respectively.
The symbol $\textbf{I}_N$ denotes the $N\times N$ identity matrix,
and  $\mathcal N\left(\textbf{x}|\boldsymbol{\mu}, \textbf{R}\right)$ stands for the probability density function (PDF) of a Gaussian random vector $\textbf{x}$ with mean $\boldsymbol{\mu}$ and covariance matrix $\textbf{R}$.
The notation
$||\textbf{x}-\textbf{y}||^2_{\textbf{W}}$ stands for
$\left(\textbf{x}-\textbf{y}\right)^T\textbf{W} \left(\textbf{x}-\textbf{y}\right)$.
The symbol $ \propto$ represents the linear scalar relationship between two real valued functions.
For Hermitian matrices $\textbf{X}$ and $\textbf{Y}$, $\textbf{X} \succeq\textbf{Y}$ ($\textbf{X} \succ \textbf{Y}$) means that $\textbf{X} - \textbf{Y}$ is positive semidefinite (definite).
The sets $\left[\textbf{A}, \textbf{B}\right]$ are defined by
$\left[\textbf {A}, \textbf{B}\right]=
\left\{\textbf{X}:\textbf{B} \succeq \textbf{X}\succeq \textbf {A}\right\}$.
The symbol
$\texttt{Bdiag}\left\{\cdot\right\}
$ stands for
block diagonal matrix with elements listed inside the bracket;
$\otimes$ denotes
the Kronecker product;
and $\textbf X_{i,j}$ denotes the component of matrix $\textbf X$ on the $i$-th row and $j$-th column.

%
\section{Problem Statement and Markov Random Field}\label{hybrid}
Consider a general connected network\footnote{A connected network is one where any two distinct agents can communicate with each other through a finite number of hops.}
of $M$  agents, with ${\mathcal{V}}=\left\{1,\ldots, M\right\}$ denoting the set of agents,
and $\mathcal{E}_{\textrm{Net}} \subset {\mathcal{V}} \times  {\mathcal{V}}$  the set of all undirected communication links in the network, i.e., if $i$ and $j$ can communicate or exchange information directly, $\left(i, j\right) \in \mathcal{E}_{\textrm{Net}}$.
At every agent $n \in \mathcal{V}$,
the local observations are given by a linear Gaussian model:
\begin{equation} \label{linear}
\textbf{y}_n = \sum_{i\in  n\cup\mathcal{I}\left(n\right)}
\textbf{A}_{n,i}\textbf{x}_i + \textbf{z}_n,
\end{equation}
where
$\mathcal{I}\left(n\right)$ denotes the set of  neighbors of agent $n$ (i.e., all agents $i$ with $\left(n,i\right) \in \mathcal{E}_{\textrm{Net}}$),
$\textbf{A}_{n,i}$ is a known coefficient matrix with full column rank,
$\textbf{x}_i$ is the local unknown parameter at agent $i$ with  dimension $N_i \times 1$ and with  prior distribution $\textbf{x}_i\sim \mathcal{N}\left(\textbf{x}_i|\textbf{0},\textbf{W}_{i}\right)$ ($\textbf{W}_{i}\succ \textbf 0$),
and $\textbf{z}_n$ is the additive noise with distribution $\textbf{z}_n\sim \mathcal{N}\left(\textbf{z}_n|\textbf{0},\textbf{R}_n\right)$, {where $\textbf R_n\succ 0$}.
It is assumed that
$p\left(\textbf{x}_i, \textbf{x}_j\right)=p\left(\textbf{x}_i\right)p\left(\textbf{x}_j\right)$
and
$p\left(\textbf{z}_i,\textbf{z}_j\right)
=p\left(\textbf{z}_i\right)p\left(\textbf{z}_j\right)$ for $i\neq j$, and the $x_{i}$'s and $z_{j}$'s  are independent for all $i$ and $j$.
The goal is to learn $\textbf{x}_i$, based on $\textbf{y}_n$, $p\left(\textbf{x}_i\right)$, and $p\left(\textbf{z}_n\right)$.\footnote{{By slightly modifying (1), the local model would  allow two neighboring agents to share a common observation and the analyses in the following sections still apply. Please refer to \citet{pairwise} for  details, and \citet{DconveCon} for the corresponding models and associated (distributed) convergence conditions.}}

In centralized estimation, all the observations $\textbf{y}_n$'s at different agents are forwarded to a central processing unit.
Define  vectors $\textbf x$, $\textbf{y}$, and $\textbf z$ as the stacking of $\textbf x_n$, $\textbf{y}_n$,  and $\textbf{z}_n$ in ascending order with respect to $n$, respectively;  then, we obtain
\begin{equation}\label{blocklinear}
\textbf{y}= \textbf{A}\textbf x + \textbf z,
\end{equation}
where
$\textbf{A}$ is constructed from $\textbf{A}_{n,i}$, with specific arrangement dependent on the network topology.
Assuming $\textbf{A}$ is of full column rank, and since
 (\ref{blocklinear}) is a standard linear model, the optimal minimum mean squared error estimate $\widehat{\textbf x}\triangleq\left[\widehat{\textbf x}_1^T,\ldots,\widehat{\textbf x}_M^T\right]^T$ of $\textbf x$   is given by  \citep{KevinMurphy}
\begin{eqnarray}\label{Central}
\widehat{\textbf x}
= \int \textbf x\frac{ p\left(\textbf x\right)p\left(\textbf{y}|\textbf x\right)}{\int p\left(\textbf x\right)p\left(\textbf{y}|\textbf x\right)\mathrm{d}\textbf x}\mathrm{d} \textbf x
=
\left(\textbf{W}^{-1}+ \textbf{A}^T\textbf{R}^{-1}\textbf{A}\right)^{-1}\textbf{A}^T\textbf{R}^{-1} \textbf{y},
\end{eqnarray}
where $\textbf{W}$ and $\textbf{R}$ are block diagonal matrices containing $\textbf{W}_i$ and $\textbf{R}_i$ as their diagonal blocks, respectively.
Although well-established, centralized estimation in large-scale networks  has several drawbacks including: 1) the transmission of $\textbf{y}_n$, $\textbf{A}_{n,i}$ and $\textbf R_n$ from
 peripheral agents to the computation center imposes large communication overhead;
2) knowledge of global network topology is  needed in order to construct $\textbf{A}$;
3) the computation burden at the computation center scales up due to the matrix inversion required in (\ref{Central})   with   complexity order $\mathcal{O}\left(\left(\sum_{i=1}^{|\mathcal{V}|}N_i\right)^3\right)$,
i.e., cubic in the dimension in general.

On the other hand, Gaussian BP running over graphical models  representing the joint posterior distribution of all $\textbf{x}_i$'s provides a distributed way to
learn
 $\textbf{x}_i$ locally, thereby mitigating the disadvantages of the centralized approach.
In particular, with Gaussian MRF, the joint distribution $p\left(\textbf x\right)p\left(\textbf{y}|\textbf x\right)$ is expressed in a pairwise form  \citep{WalkSum1}:
\begin{equation}\label{MRF-eqn}
p\left(\textbf x\right)p\left(\textbf{y}|\textbf x\right)
=
\prod_{n\in \mathcal{V}} \psi_{n} \left(\textbf x_n, \left\{\textbf{y}_i\right\}_{i\in \left\{n\cup\mathcal{I}\left(n\right)\right\} }\right)
\prod_{\left(n,i\right)\in \mathcal{E}_{\textrm{MRF}} } \psi_{n,i} \left(\textbf x_n, \textbf{x}_i\right),
\end{equation}
where
\begin{equation}\label{def}
\mathcal{E}_{\textrm{MRF}} \triangleq \mathcal{E}_{\textrm{Net}} \cup \left\{\left(n,i\right) | \exists k,  k\neq n, k\neq i,  \text{such that}  \ \left(n,k\right) \in \mathcal{E}_{\textrm{Net}},  \text{and}\ \left(i,k\right) \in \mathcal{E}_{\textrm{Net}}\right\};
\end{equation}
\begin{equation}\label{MRF-local}
 \psi_{n} \left(\textbf x_n, \left\{\textbf{y}_i\right\}_{i\in n\cup\mathcal{I}\left(n\right) }\right)
=\exp\left\{\-\frac{1}{2}\left(
\textbf x_n^T \textbf{W}_n^{-1}\textbf x_n
 + \sum_{i\in n\cup\mathcal{I}\left(n\right) }\textbf{y}_i^T\textbf{R}_i^{-1}\textbf x_n
\right) \right\}
\end{equation}
is the  potential function at agent $n$, and
\begin{equation}\label{MRF-pair}
\begin{split}
\psi_{n,i} (\textbf x_n, \textbf{x}_i)
=&\exp -\bigg\{\frac{1}{2}
\big[(\textbf{A}_{n,n}\textbf x_n)^T\textbf R_n^{-1}(\textbf{A}_{n,i}\textbf{x}_i)
+(\textbf{A}_{i,n}\textbf x_n)^T\textbf{R}_i^{-1}(\textbf{A}_{i,i}\textbf{x}_i)\\
&
+
 \sum_{\substack{k\in \{\widetilde{k}|(\widetilde{k},i)\in \mathcal{E}_{\textrm{Net}}, \\ (\widetilde{k},n)\in \mathcal{E}_{\textrm{Net}}\} }}
 (\textbf{A}_{k,n}\textbf x_n)^T\textbf{R}_k^{-1}(\textbf{A}_{k,i}\textbf{x}_i)\big]
 \bigg\}
\end{split}
\end{equation}
is the edge potential between $\textbf x_n$ and $\textbf{x}_i$.
After setting up the graphical model representing the joint distribution  in  (\ref{MRF-eqn}),
messages are exchanged between  pairs of agents $n$ and $i$ with $\left(n,i\right)\in
\mathcal{E}_{\textrm{MRF}} $.
More specifically, according to the standard derivation of Gaussian BP, at the $\ell$-${\textrm{th}}$ iteration, the message passed from agent $n$ to agent $i$ is
\begin{equation}\label{MRF1}
w^{\left(\ell\right)}_{n\to i}\left(\textbf{x}_i\right) =
\int  \psi_{n} \left(\textbf x_n, \left\{\textbf{y}_k\right\}_{k\in n\cup\mathcal{I}\left(n\right) }\right)
\psi_{n,i}\left(\textbf x_n,\textbf{x}_i\right)\prod_{k\in \mathcal I\left(n\right)\setminus i}w^{\left(\ell-1\right)}_{k\to n}\left(\textbf x_n\right)\mathrm{d}\textbf x_n.
\end{equation}

As shown by (\ref{MRF1}),   Gaussian BP is iterative  with each agent alternatively receiving messages from its neighbors and forwarding out  updated  messages.  At each iteration,  agent $i$  computes its belief on  variable $\textbf{x}_i$ as
\begin{equation}\label{MRF2}
b_{\textrm{MRF}}^{\left(\ell\right)}\left(\textbf{x}_i\right)\propto
\psi_i\left(\textbf{x}_i, \left\{\textbf{y}_n\right\}_{n\in i\cup\mathcal{I}\left(i\right) }\right)\prod_{k\in \mathcal I\left(n\right)}w^{\left(\ell\right)}_{k\to i}\left(\textbf{x}_i\right).
\end{equation}
It is known that,  as the messages (\ref{MRF1}) converge, the mean of the belief (\ref{MRF2}) is the exact mean of the marginal distribution of $\textbf{x}_i$ \citep{DiagnalDominant}.

It might seem that our distributed inference problem is now solved, as a solution is readily available.
However, there are two serious limitations for the Gaussian MRF approach.

First, messages are passed between pairs of agents in $\mathcal{E}_{\textrm{MRF}}$, which according to the definition  (\ref{def}) includes not only those direct neighbors, but also pairs that are two hops away but share a common neighbor.
This is illustrated in Fig.~1, where Fig.~1(a) shows a network of $4$ agents {with a line between two neighboring agents indicating} the availability of a physical communication link, and Fig.~1(b) shows the equivalent pairwise graph.
For this example, in the physical network, there is no direct connection between agents $1$ and $4$, nor between agents $1$ and $3$.
But in the pairwise representation, those connections are present.  We summarize the above observations in the following remark.

\begin{remark}\label{graph}
For a network with communication edge set $\mathcal{E}_{\textit{Net}}$ and local observations following (\ref{linear}),  the corresponding MRF graph edge set satisfies
$\mathcal{E}_{\textit{{MRF}}}\supseteq \mathcal{E}_{\textit{{Net}}}$.
{Thus, Gaussian BP for Gaussian MRFs cannot be applied to the distributed inference  problem with the local observation model (1).}\footnote{{ In Section 5, we further show that the convergence condition of Gaussian BP obtained in this paper for model (1) encompasses all existing convergence conditions of Gaussian BP for the corresponding Gaussian MRF. }}
\end{remark}

The consequence of the above findings is that, not only does information need to be shared among agents two hops away from each other to construct the edge potential function in (\ref{MRF-pair}), but also the  messages (\ref{MRF1}) may be required to be exchanged among non-direct neighbors, where a physical communication link is not available.  This complicates significantly the message exchange scheduling.

\begin{figure}\label{SystemModel}
  \centering
\mbox{\subfigure[]{\epsfig{figure=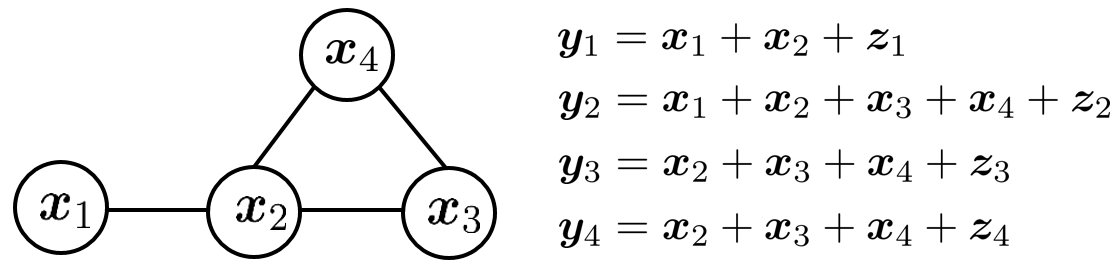,width=3.3in}}\label{Network} }\\
\mbox{ \subfigure[]{\epsfig{figure=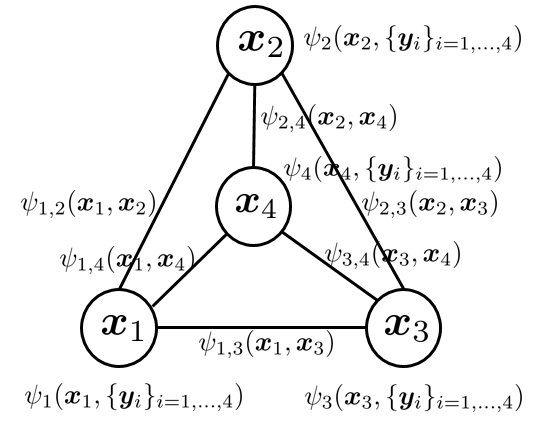,width=1.9in}}\label{MRF} }\\
\mbox{ \subfigure[]{\epsfig{figure=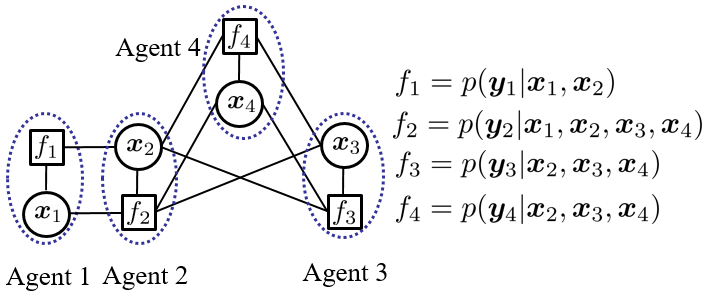,width=3.3in}}\label{FG} }\\
\caption{(a) A physical network with 4 agents, where $\left\{1,2\right\}$ and $\left\{2,3,4\right\}$ are two groups of agents that are within the communication range of each other, respectively. $\textbf{x}_i$ is the local unknown vector, and $\textbf{y}_i$ is the local observation at agent $i$ that follows (\ref{linear});
(b) The corresponding MRF of Fig.~1 (a) with $ \psi_{n} \left(\textbf x_n, \left\{\textbf{y}_i\right\}_{i\in n\cup\mathcal{I}(n) }\right)$ and $\psi_{n,i} \left(\textbf x_n, \textbf{x}_i\right)$ defined in
(\ref{MRF-local}) and (\ref{MRF-pair}), respectively.
(c) The corresponding factor graph of Fig.~1 (a) with $f_i$ defined in (\ref{jointpost}).
Since $p(\textbf{x}_i)$ does not involve  message passing, the $p(\textbf{x}_i)$ associated to each variable node is not drawn to keep the figure simple.}
\end{figure}

Secondly, even if the message scheduling between non-neighboring agents can be realized, the convergence of (\ref{MRF1}) is not guaranteed in loopy networks.  For Gaussian MRF with scalar variables, sufficient convergence conditions have been proposed in \citep{DiagnalDominant,WalkSum1,Suqinliang}.  However, depending on how the factorization of the underlying joint Gaussian distribution is performed, Gaussian BP may exhibit different convergence properties as different factorizations (different Gaussian models)  lead to fundamentally different recursive update structures.
Furthermore, these results apply only to scalar Gaussian BP, and extension to vector-valued Gaussian BP is nontrivial as we show in this paper.

The next section derives distributed vector inference based on Gaussian BP with high order interactions (beyond pairwise connections), where information sharing and message exchange requirement conform to the physical network topology.  Furthermore,  convergence conditions will be studied in Section \ref{analysis}, and we show in Section \ref{relationship} that the convergence condition obtained is strictly weaker than, i.e., subsumes the convergence conditions in \citep{DiagnalDominant,WalkSum1,Suqinliang}.


\section{Distributed Inference with Vector-Valued Gaussian BP and Non-Pairwise Interaction}\label{vmpgamma}

The joint distribution  $p\left(\textbf x\right)p\left(\textbf{y}|\textbf x\right)$ is first written as the product of the  prior distribution and the likelihood function of each local linear Gaussian model in (\ref{linear}) as
\begin{equation}\label{jointpost}
p\left(\textbf x\right)p\left(\textbf{y}|\textbf x\right) =
\prod_{n\in \mathcal{V}}
p\left(\textbf x_n\right)
\prod_{n\in \mathcal{V}}
\underbrace{p\left(\textbf{y}_n| \left\{ \textbf{x}_i\right\}_{i\in  n\cup\mathcal{I}\left(n\right)} \right)}_{\triangleq f_n }.
\end{equation}
To facilitate the derivation of the distributed inference algorithm,
the factorization in  (\ref{jointpost})
is expressed in terms of a
factor graph  \citep{Kschischang},
where  every vector variable  $\textbf{x}_i$ is represented by a circle (called variable node)
and the probability distribution of a vector variable or a group of vector variables is represented by a square (called factor node).
A variable node is connected to a factor node if the variable is involved in that particular factor.
For example, Fig.~1(c) shows the factor graph representation for the network in Fig.~1(a).

We derive the Gaussian
BP algorithm   over the corresponding factor graph to learn $\textbf x_n$ for all $n\in \mathcal V$ \citep{Kschischang}. It
involves two types of messages:
one is the message
from a variable node $\textbf x_j$ to its neighboring factor node $f_n$, defined as
\begin{equation} \label{BPv2f1}
m^{\left(\ell\right)}_{j \to f_n}\left(\textbf x_j\right)
= p\left(\textbf x_j\right)
\prod_{f_k\in \mathcal B(j)\setminus f_n}m^{\left(\ell-1\right)}_{f_k\to j}\left(\textbf x_j\right),
\end{equation}
where $\mathcal B\left(j\right)$ denotes the set of neighbouring factor nodes  of $\textbf x_j$,
and $m^{\left(\ell-1\right)}_{f_k\to j}\left(\textbf x_j\right)$ is the   message  from $f_k$ to $\textbf x_j$ at time $\ell-1$.
The second type of
message is from a factor node $f_n$ to a neighboring variable node $\textbf{x}_i$, defined as
\begin{equation}\label{BPf2v1}
m^{\left(\ell\right)}_{f_n \to i}\left(\textbf{x}_i\right)
=  \int \cdots \int
f_n \times \!
\prod_{j\in\mathcal B\left(f_n\right)\setminus i} m^{ \left(\ell\right)}_{j \to f_n}\left(\textbf x_j\right)
\,\mathrm{d}\left\{\textbf x_j\right\}_{j\in\B\left(f_n\right)\setminus i},
\end{equation}
where $\mathcal B\left(f_n\right)$ denotes the set of neighboring variable nodes of $f_n$.
The process iterates between equations (\ref{BPv2f1}) and (\ref{BPf2v1}).
At each iteration $\ell$, the approximate marginal distribution, also referred to as belief, on $\textbf{x}_i$ is computed locally at $\textbf{x}_i$ as
\begin{equation} \label{BPbelief}
b_{\textrm{BP}}^{\left(\ell\right)}\left(\textbf{x}_i\right)
 = p\left(\textbf{x}_i\right) \prod_{ f_n\in \mathcal B\left(i\right)} m^{\left(\ell\right)}_{ f_n \to i}\left(\textbf{x}_i\right).
\end{equation}

In the sequel, we  derive the exact expressions for the messages
$m^{\left(\ell\right)}_{j \to f_n}\left(\textbf x_j\right)$,
$m^{\left(\ell\right)}_{f_n \to i}\left(\textbf{x}_i\right)$, and belief $b_{\textrm{BP}}^{\left(\ell\right)}\left(\textbf{x}_i\right)$.
First, let
the initial  messages  at each variable node and factor node  be
in Gaussian function  forms as
\begin{equation}\label{initial}
m^{\left(0\right)}_{f_n \to i}
\left(\textbf{x}_i\right)\propto
\exp
\left\{-\frac{1}{2}
||\textbf{x}_i- \textbf{v}^{\left(0\right)}_{f_n\to i}||^2
_{ \textbf{J}^{\left(0\right)}_{f_n\to i} }
\right\}.
\end{equation}
In   Appendix A,
it is shown that the general expression for the message from variable node $j$ to factor node $f_n$ is
\begin{equation} \label{BPvs2f1}
m^{\left(\ell\right)}_{j \to f_n}\left(\textbf x_j\right) \propto
\exp
\left\{-\frac{1}{2}
||\textbf x_j- \textbf{v}^{\left(\ell\right)}_{j\to f_n}||^2
_{\textbf{J}^{\left(\ell\right)}_{j\to f_n}}
\right\},
\end{equation}
with
\begin{equation} \label{v2fV}
\textbf{J}^{\left(\ell\right)}_{j \to f_n}
= \textbf{W}_j^{-1} +
\sum_{f_k\in\B\left(j\right)\setminus f_n}
 \textbf{J}_{f_k\to j}^{\left(\ell-1\right)},
\end{equation}
\begin{equation}\label{v2fm}
\textbf{v}^{\left(\ell\right)}_{j\to f_n}=
\left[\textbf{J}^{\left(\ell\right)}_{j\to f_n}\right]^{-1}
\left[
\sum_{f_k\in\B\left(j\right)\setminus f_n}
\textbf{J}_{f_k\to j}^{\left(\ell-1\right)}
\textbf{v}^{\left(\ell-1\right)}_{f_k\to j}\right],
\end{equation}
where $\textbf{J}_{f_k\to j}^{\left(\ell-1\right)}$ and $ \textbf{v}_{f_k\to j}^{\left(\ell-1\right)}$ are the message information matrix (inverse of covariance matrix) and mean vector  received at variable node $j$ at the $\left(\ell-1\right)$-$\textrm{th}$ iteration, respectively.
Furthermore,
the message  from factor node  $f_n$ to variable node $i$ is given by
\begin{equation} \label{f2v}
m^{\left(\ell\right)}_{f_n \to i}\left(\textbf{x}_i\right)\propto
\alpha_{f_n \to i}^{\left(\ell\right)}
\exp
\left\{-\frac{1}{2}
||\textbf{x}_i- \textbf{v}^{\left(\ell\right)}_{f_n\to i}||^2
_{\textbf{J}^{\left(\ell\right)}_{f_n\to i}}
\right\},
\end{equation}
with
\begin{equation}\label{Cov}
\begin{split}
\textbf{J}^{\left(\ell\right)}_{f_n\to i}
=
\textbf{A}_{n,i}^T
\left[ \textbf{R}_n
+
\sum_{j\in\B\left(f_n\right)\setminus i} \textbf{A}_{n,j}
\left[\textbf{J}^{\left(\ell\right)}_{j\to f_n}\right]^{-1}\textbf{A}_{n,j}^T \right]^{-1}
\textbf{A}_{n,i},
\end{split}
\end{equation}

\begin{equation}\label{f2vmm}
\begin{split}
\textbf{v}^{\left(\ell\right)}_{f_n\to i}
=
\left[\textbf{J}_{f_n\to i}^{\left(\ell\right)}\right]^{-1}
\textbf{A}_{n,i}^T \left[\textbf{R}_n
+
\sum_{j\in\B\left(f_n\right)\setminus i} \textbf{A}_{n,j}
\left[\textbf{J}^{\left(\ell\right)}_{j\to f_n}\right]^{-1}\textbf{A}_{n,j}^T \right]^{-1}
\left(\textbf{y}_n-\sum_{j\in\B\left(f_n\right)\setminus i} \textbf{A}_{n,j}
\textbf{v}^{\left(\ell\right)}_{j\to f_n}\right),
\end{split}
\end{equation}
and
\begin{equation}\label{eigintegral-1}
\alpha_{f_n \to i}^{\left(\ell\right)}
\propto
\int\ldots \int \exp\left\{-\frac{1}{2}
    \textbf z^T
    \boldsymbol{\Lambda}_{f_n \to i}^{\left(\ell\right)}
    \textbf z \right\}    \, \mathrm{d}\textbf z.
\end{equation}
In (\ref{eigintegral-1}), $\boldsymbol{\Lambda}_{f_n \to i}^{\left(\ell\right)}$ is
a diagonal matrix containing the eigenvalues of
$\textbf{A}_{n,\left\{\B\left(f_n\right)\setminus i\right\}}^T
\textbf{R}_n^{-1}\textbf{A}_{n,\left\{\B\left(f_n\right)\setminus i\right\}}
+ \textbf{J}^{\left(\ell\right)}_{\left\{\B\left(f_n\right)\setminus i\right\}\to f_n}$,
with  $\textbf{A}_{n,\left\{\B\left(f_n\right)\setminus i\right\}}$ denoting a row block matrix containing $\textbf{A}_{n,j}$ as row elements for all $j\in \B\left(f_n\right)\setminus i$ arranged in ascending order,
and $\textbf{J}^{\left(\ell\right)}_{\left\{\B\left(f_n\right)\setminus i\right\}\to f_n}$ denoting a block diagonal matrix with $\textbf{J}^{\left(\ell\right)}_{j\to f_n}$ as its block diagonal elements
for all $j\in \B\left(f_n\right)\setminus i$ arranged in ascending order.

Obviously, the  validity of (\ref{f2v}) depends on the existence of $\alpha_{f_n \to i}^{\left(\ell\right)} $.
It is evident that (\ref{eigintegral-1}) is the integral of a Gaussian distribution  and equals to a constant  when $\boldsymbol{\Lambda}_{f_n \to i}^{\left(\ell\right)}
\succ \textbf{0}$ or equivalently $\textbf{A}_{n,\left\{\B\left(f_n\right)\setminus i\right\}}^T\textbf{R}_n^{-1}\textbf{A}_{n,\left\{\B\left(f_n\right)\setminus i\right\}}
+ \textbf{J}^{ \left(\ell\right)}_{\left\{\B\left(f_n\right)\setminus i\right\}\to f_n}\succ \textbf{0}.$
Otherwise,  $\alpha_{f_n \to i}^{\left(\ell\right)}$ does not exist.
Therefore, the necessary and sufficient condition for the  existence of  $m^{\left(\ell\right)}_{f_n \to i}\left(\textbf{x}_i\right)$  is
\begin{equation} \label{suffici}
\textbf{A}_{n,\left\{\B\left(f_n\right)\setminus i\right\}}^T
\textbf{R}_n^{-1}\textbf{A}_{n,\left\{\B\left(f_n\right)\setminus i\right\}}
+\textbf{J}^{ \left(\ell\right)}_{\left\{\B\left(f_n\right)\setminus i\right\}\to f_n} \succ \textbf{0}.
\end{equation}
In general, the necessary and sufficient condition is difficult to be verified, as
$\textbf{J}^{ \left(\ell\right)}_{\left\{\B\left(f_n\right)\setminus i\right\}\to f_n}$ changes in each iteration.
However, as $\textbf{R}_n^{-1}\succ \textbf{0}$, it can be decomposed as $\textbf{R}_n^{-1}= \widetilde{\textbf{R}}^T_n\widetilde{\textbf{R}}_n$.
Then $$\textbf{A}_{n,\left\{\B\left(f_n\right)\setminus i\right\}}^T
\textbf{R}_n^{-1}\textbf{A}_{n,\left\{\B\left(f_n\right)\setminus i\right\}}
=
\left(\widetilde{\textbf{R}}_n\textbf{A}_{n,\left\{\B\left(f_n\right)\setminus i\right\}}\right)^T
\left(\widetilde{\textbf{R}}_n\textbf{A}_{n,\left\{\B\left(f_n\right)\setminus i\right\}}\right) \succeq \textbf{0}.
$$
Hence, one simple sufficient condition  to guarantee (\ref{suffici})
 is
$\textbf{J}^{\left(\ell\right)}_{\left\{\B\left(f_n\right)\setminus i\right\}\to f_n}\succ \textbf{0}$ or equivalently its diagonal block matrix $\textbf{J}^{\left(\ell\right)}_{j\to f_n}\succ \textbf{0}$ for all $j\in \B\left(f_n\right)\setminus i$.
The following lemma shows that  setting the initial message covariances
$\textbf{J}_{f_n\to i}^{\left(0\right)}\succeq \textbf{0}$  for all $\left(n, i\right)\in \mathcal{E}_{\textrm{Net}}$
guarantees  $\textbf{J}^{\left(\ell\right)}_{j\to f_n}\succ \textbf{0}$ for $\ell \geq 1$ and all $\left(n,j\right) \in \mathcal{E}_{\textrm{Net}}$.

\begin{lemma}\label{pdlemma}
Let the initial messages at factor node $f_k$ be in Gaussian  forms with the initial message information matrix
$\textbf{J}_{f_k\to j}^{\left(0\right)} \succeq \textbf{0}$ for all $k \in \mathcal{V}$ and $j \in \mathcal{B}\left(f_k\right)$.
Then
$\textbf{J}_{j\to f_n}^{\left(\ell\right)} \succ \textbf{0}$
 and $\textbf{J}^{\left(\ell\right)}_{f_k \to j} \succ \textbf{0}$
 for all $\ell\geq 1$ with $j \in \mathcal{V}$ and
$f_n, f_k \in \mathcal{B}\left(j\right)$.
{Furthermore,
in this case,
all  messages $m^{\left(\ell\right)}_{j \to f_n}\left(\textbf x_j\right)$ and
$m^{\left(\ell\right)}_{f_k \to j}\left(\textbf{x}_i\right)$ are well defined.}
\end{lemma}
\begin{proof}
See Appendix B.
\end{proof}

For this factor graph based approach, according to the message updating procedure  (\ref{BPvs2f1}) and (\ref{f2v}),  message exchange is only needed between neighboring agents {(an agent refers to a variable-factor pair as shown in Fig.~1 (c))}.
For example, the messages transmitted from agent $n$ to its neighboring agent $i$ are   $m_{f_n\to i}^{\left(\ell\right)}\left(\textbf{x}_i\right)$ and $m_{n\to f_i}^{\left(\ell\right)}\left(\textbf x_n\right)$.
{Thus, the factor graph does impose a clear messaging schedule, and
the message passing scheme given in (\ref{BPv2f1}) and (\ref{BPf2v1}) conforms with the network topology.}
Furthermore, if the messages $m_{j\to f_n}^{\left(\ell\right)}\left(\textbf x_j\right)$
and $m_{f_n\to i}^{\left(\ell\right)}\left(\textbf{x}_i\right)$ exist for all $\ell$
(which can be achieved using Lemma \ref{pdlemma}),
the messages  are Gaussian, therefore only the corresponding mean vectors and information matrices (inverse of covariance matrices) are needed to be exchanged.

Finally, if the Gaussian BP messages exist, according to the definition of belief in (\ref{BPbelief}),
$b_{\textrm{BP}}^{\left(\ell\right)}\left(\textbf{x}_i\right)$ at  iteration  $\ell$ is computed as
\begin{equation} \label{belief2}
\begin{split}
\mathbf{b}_{\textrm{BP}}^{\left(\ell\right)}\left(\textbf{x}_i\right)
&=p\left(\textbf{x}_i\right)\prod_{f_n\in\mathcal B\left(i\right)} m_{f_n\to i}^{\left(\ell\right)}\left(\textbf{x}_i\right),\\
&\propto  \mathcal{N}\left(\textbf{x}_i|
\boldsymbol{\mu}_i^{\left(\ell\right)}, \textbf{P}_i^{\left(\ell\right)}\right),\nonumber
\end{split}
\end{equation}
where the belief covariance matrix
\begin{equation}  \label{beliefcov}
\textbf{P}_i^{\left(\ell\right)} =
\left[\textbf{W}_i^{-1}
+\sum_{f_n\in\B\left(i\right)}
\textbf{J}_{f_n\to i}^{\left(\ell\right)}\right]^{-1},
\end{equation}
and mean vector
\begin{equation} \label{beliefmean}
\boldsymbol{\mu}_i^{\left(\ell\right)}=\textbf{P}_i^{\left(\ell\right)}\left[
\sum_{f_n\in\B\left(i\right)}
\textbf{J}_{f_n\to i}^{\left(\ell\right)}\textbf{v}^{\left(\ell\right)}_{f_n\to i}\right].
\end{equation}

{The iterative algorithm based on Gaussian BP is summarized as follows.
The algorithm is started by setting the messages from factor nodes to variable nodes as in
(\ref{initial}).
At each round of message exchange, every variable node computes the output messages to its neighboring factor nodes
according to (\ref{v2fV}) and (\ref{v2fm}).
After receiving the messages from its neighboring variable nodes, each factor node computes its output messages according to (\ref{Cov}) and (\ref{f2vmm}).
The iterative computation terminates when the iterates in  (\ref{BPvs2f1}) or  (\ref{f2v}) tend to approach  a fixed value or the maximum number of iterations is reached.}

\begin{remark}
{ 
We assume  that $\textbf R_n\succ \textbf 0$ in this paper.
If, however, some of the observations are noiseless, for example,
 $\textbf R_n=\textbf 0$, the local observation is
$\textbf{y}_n = \sum_{i\in  n\cup\mathcal{I}\left(n\right)}
\textbf{A}_{n,i}\textbf{x}_i$.
Then the corresponding local likelihood function is represented by the Dirac measure
$\delta(\textbf{y}_n - \sum_{i\in  n\cup\mathcal{I}\left(n\right)}
\textbf{A}_{n,i}\textbf{x}_i)$.
 Suppose, for example, there is only one agent with $\textbf R_n=\textbf 0$, and all others are $\textbf R_i\succ\textbf 0$.
The the joint distribution is written as
\begin{equation}
p\left(\textbf x\right)p\left(\textbf{y}|\textbf x\right) =
\delta(\textbf{y}_n - \sum_{i\in  n\cup\mathcal{I}\left(n\right)}
\textbf{A}_{n,i}\textbf{x}_i)
\prod_{j\in \mathcal{V}}
p\left(\textbf x_j\right)
\prod_{k\in \mathcal{V}}
{p\left(\textbf{y}_k| \left\{ \textbf{x}_i\right\}_{i\in  k\cup\mathcal{I}\left(k\right)} \right)}.\nonumber
\end{equation}
In this case, if $\textbf A_{n,n}$ is  invertible, then, by the definition of the Dirac measure, we have
 $\textbf{x}_n = \textbf{A}_{n,n}^{-1}\left(\textbf{y}_n - \sum_{i\in  \mathcal{I}\left(n\right)}
\textbf{A}_{n,i}\textbf{x}_i\right)$.
By substituting this equation into all of the likelihood functions  involving $\textbf x_n$, we have the equivalent joint distribution as in (\ref{jointpost}) with all the likelihood functions having a positive definite noise covariance.
We thereafter can apply Gaussian BP to this new factorization and the convergence analysis in this paper still applies.
Therefore, without loss of generality, we assume all $\textbf R_n\succ \textbf 0$.
Note that when  $\textbf R_n=\textbf 0$ for all $n$, this problem is equivalent to solving   algebraic equations, which has been studied in \citep{linearequations} using Gaussian BP.}
\end{remark}

%
%
\section{Convergence Analysis}\label{analysis}
The challenge of deploying the Gaussian BP algorithm for large-scale networks is in determining whether it will converge or not.
In particular, it
is generally known that if the factor graph contains cycles, the Gaussian BP algorithm may
diverge.
Thus, determining  convergence conditions for the Gaussian BP algorithm is very important.
Sufficient conditions for the convergence of Gaussian BP with scalar variables in loopy graphs are available in \citep{DiagnalDominant, WalkSum1,Suqinliang}.
 However,  these conditions are derived based on  pairwise graphs with local functions  in
 the form of (\ref{MRF-local}) and (\ref{MRF-pair}).
This  contrasts with the model considered in this paper, where the
$f_n$ in (\ref{jointpost}) involves high-order interactions between vector variables, and
thus the convergence results in \citep{DiagnalDominant, WalkSum1, Suqinliang} cannot be applied to the factor graph based vector-form Gaussian BP.

Due to the recursive  updating property of  $m_{j\to f_n}^{\left(\ell\right)}\left(\textbf x_j\right)$ and $m_{f_n\to i}^{\left(\ell\right)}\left(\textbf{x}_i\right)$ in (\ref{BPvs2f1}) and (\ref{f2v}), the message evolution can be simplified by combining these two kinds of messages into one.
By substituting
 $ \textbf{J}^{\left(\ell\right)}_{j \to f_n}$ in (\ref{v2fV}) into   (\ref{Cov}), the updating of the message covariance matrix inverse, referred to as message information matrix in the following, can be denoted as
\begin{eqnarray}\label{CovFunc}
\textbf{J}_{f_n\to i}^{\left(\ell\right)}
&=&
\textbf{A}_{n,i}^T \bigg[\textbf{R}_n
+ \sum_{j\in\B\left(f_n\right)\setminus i} \textbf{A}_{n,j}
\bigg[
 \textbf{W}_{j}^{-1} +
\sum_{f_k\in\B\left(j\right)\setminus f_n}
\textbf{J}_{f_k\to j}^{\left(\ell-1\right)}
\bigg]^{-1}
\textbf{A}_{n,j}^T \bigg]^{-1}
\textbf{A}_{n,i}\nonumber\\
&\triangleq&
\mathcal{F}_{n\to i}
\left(\left\{
\textbf{J}_{f_k\to j}^{\left(\ell-1\right)}\right\}_{\left(f_k, j\right)\in \mathcal{\widetilde{B}}\left(f_n, i\right)}
  \right),
\end{eqnarray}
where $\mathcal{\widetilde{B}}\left(f_n, i\right)=\left\{\left(f_k, j\right) | j \in \B\left(f_n\right)\setminus i,  f_k\in \B\left(j\right)\setminus f_n\right\}$.
Observing that $\textbf{J}_{f_n\to i}^{\left(\ell\right)}$ in (\ref{CovFunc}) is independent of $\textbf{v}^{\left(\ell\right)}_{j\to f_n} $ and $\textbf{v}^{\left(\ell\right)}_{f_n\to i}$ in (\ref{v2fm}) and (\ref{f2v}), so
we can first focus on the convergence property of $\textbf{J}_{f_n\to i}^{\left(\ell\right)}$ alone and then later on that of $\textbf{v}^{\left(\ell\right)}_{f_n\to i}$.
With the convergence characterization of
$\textbf{J}_{f_n\to i}^{\left(\ell\right)}$ and
$\textbf{v}^{\left(\ell\right)}_{f_n\to i}$,  we will further investigate the convergence  of belief covariances  and means in (\ref{beliefcov}) and (\ref{beliefmean}), respectively.

{ Note that computing $\textbf P_{j}^{(\ell)}$ requires all the incoming messages from neighboring nodes including $\textbf{J}_{f_n\to j}^{\left(\ell\right)}$ as shown in
	(\ref{beliefcov}) by replacing the subscript $i$ with $j$ in (\ref{beliefcov}). However, according to (\ref{CovFunc}), when computing  $\textbf{J}_{f_n\to i}^{\left(\ell\right)}$ the quantity $\textbf{J}_{f_n\to j}^{\left(\ell-1\right)}$ is excluded, i.e., the quantity inside the inner square brackets equals $[\textbf{P}_j^{\left(\ell-1\right)}]^{-1} - \textbf{J}_{f_n\to j}^{\left(\ell-1\right)}$.
	Therefore, one cannot compute $\textbf{J}_{f_n\to i}^{\left(\ell\right)}$ from $\textbf P_{j}^{(\ell)}$ alone.}

\subsection{Convergence of Message Information Matrices}\label{Convariance}
To efficiently represent the updates of all message
information matrices, we {{introduce} the following definitions.
Let
$${\textbf{J}}^{\left(\ell-1\right)}
\triangleq
\texttt{Bdiag}
\left(\left\{\textbf{J}_{f_n\to i}^{\left(\ell-1\right)}\right\}_{n\in \mathcal{V},i\in \B\left(f_n\right)}\right)$$ be
a block diagonal  matrix with diagonal blocks
being the   message information matrices in the network at time $\ell-1$
with index arranged in ascending order first on $n$ and then on $i$.
Using the definition of $\textbf{J}^{\left(\ell-1\right)}$, the term $\sum_{f_k \in \mathcal B\left(j\right) \backslash f_n} \textbf{J}_{f_k\rightarrow j}^{\left(\ell-1\right)}$ in (\ref{CovFunc}) can be written as $\boldsymbol{\Xi}_{n,j} \textbf{J}^{\left(\ell-1\right)} \boldsymbol{\Xi}_{n,j}^T$, where $\boldsymbol{\Xi}_{n,j}$ is for selecting appropriate components from $\textbf{J}^{\left(\ell-1\right)}$ to form the summation.
Further, define $\textbf{H}_{n,i}=\left[\left\{ \textbf{A}_{n,j} \right\}_{j\in B\left(f_n\right) \backslash i}\right]$,
$\boldsymbol{\Psi}_{n,i}  = \texttt{Bdiag} \left( \left\{\textbf{W}_j ^{-1} \right\}_{j\in B\left(f_n\right) \backslash i} \right)$ and
$\textbf{K}_{n,i}=\texttt{Bdiag} \left(\left\{ \boldsymbol{\Xi}_{n,j} \right\}_{j\in B\left(f_n\right) \backslash i}\right) $, all with component blocks arranged with ascending order on $j$.  Then (\ref{CovFunc}) can be written as
\begin{equation}\label{CovFunc3}
    \textbf{J}^{\left(\ell\right)}_{f_n\rightarrow i}=\textbf{A}_{n,i}^T\left\{\textbf{R}_n+ \textbf{H}_{n,i}\left[\boldsymbol \Psi_{n,i} + \textbf{K}_{n,i} \left(\textbf{I}_{|\mathcal{B}\left(f_n\right)|-1} \otimes \textbf{J}^{\left(\ell-1\right)}\right) \textbf{K}_{n,i}^T \right]^{-1} \textbf{H}_{n,i}^T \right\} ^{-1}\textbf{A}_{n,i}.
\end{equation}

Now, we define the function $\mathcal{F}\triangleq\left\{\mathcal{F}_{1\to k}, \ldots, \mathcal{F}_{n\to i}, \ldots, \mathcal{F}_{n \to M}\right\}$ that satisfies
${\textbf{J}}^{\left(\ell\right)} = \mathcal{F}\left({\textbf{J}}^{\left(\ell-1\right)}\right) $.
Then, by stacking $\textbf{J}_{f_n\to i}^{\left(\ell\right)}$ on the left side of
(\ref{CovFunc3}) for all $n$ and $i$ as the block diagonal matrix $\textbf{J}^{\left(\ell\right)}$, we obtain
\begin{eqnarray}\label{CovFunc5}
  \textbf{J}^{\left(\ell\right)}
  &=& \textbf{A}^T \big \{ \boldsymbol{\Omega}+ \textbf{H}\left[\boldsymbol{\Psi} + \textbf{K} \left(\mathbf{I}_\varphi \otimes \textbf{J}^{\left(\ell-1\right)}\right) \textbf{K}^T \right]^{-1} \textbf{H}^T \big\} ^{-1}\textbf{A}, \nonumber\\
   &\triangleq& \mathcal F\left(\textbf{J}^{\left(\ell-1\right)}\right),
\end{eqnarray}
where $\textbf{A}$, $\textbf{H}$,
$\boldsymbol{\Psi}$,  and $\textbf{K}$ are block diagonal matrices with block elements $\textbf{A}_{n,i}$, $\textbf{H}_{n,i}$, $\boldsymbol{\Psi}_{n,i} $, and $\textbf{K}_{n,i}$, respectively, arranged in ascending order, first on $n$ and then on $i$ (i.e., the same order as $\textbf{J}^{\left(\ell\right)}_{f_n \rightarrow i}$ in $\textbf{J}^{\left(\ell\right)}$).
Furthermore,
$\varphi={\sum _{n=1} ^M |\mathcal B\left(f_n\right)|\left(|\mathcal B\left(f_n\right)|-1\right)}$
and
$\boldsymbol{\Omega}$ is a block diagonal matrix with diagonal blocks $\textbf{I} _{|B\left(f_n\right)|} \otimes \textbf{R}_n$ with ascending order on $n$.
We first present  some properties of the updating operator $\mathcal{F}\left(\cdot\right)$,
the  proofs   being provided in Appendix C.

\begin{proposition} \label{P_FUN}
 The updating operator $\mathcal{F}\left(\cdot\right)$ satisfies the following properties:
\end{proposition}

\noindent P \ref{P_FUN}.1:
$\mathcal{F}\left(\textbf{J}^{\left(\ell\right)}\right) \succeq \mathcal{F}\left(\textbf{J}^{\left(\ell-1\right)}\right)$, if $\textbf{J}^{\left(\ell\right)} \succeq \textbf{J}^{\left(\ell-1\right)}\succeq \textbf{0}$.

\noindent P \ref{P_FUN}.2: $\alpha\mathcal{F}\left(\textbf{J}^{\left(\ell\right)}\right) \succ  \mathcal{F}\left(\alpha \textbf{J}^{\left(\ell\right)}\right)$
and
$\mathcal{F}\left(\alpha^{-1}\textbf{J}^{\left(\ell\right)}\right) \succ  \alpha^{-1}\mathcal{F}\left(\textbf{J}^{\left(\ell\right)}\right)$, if $\textbf{J}^{\left(\ell\right)} \succ \textbf{0}$ and $\alpha>1$.

\noindent P \ref{P_FUN}.3:
Define
$\textbf{U}\triangleq \textbf{A}^T  \boldsymbol{\Omega}^{-1}\textbf{A}$
and $\textbf{L}\triangleq
\textbf{A}^T \left[  \boldsymbol{\Omega}+ \textbf{H}\boldsymbol{\Psi}^{-1} \textbf{H}^T \right] ^{-1}\textbf{A}$.
With arbitrary $\textbf{J}^{\left(0\right)}\succeq \textbf{0}$,
$\mathcal{F}\left(\textbf{J}^{\left(\ell\right)}\right)$ is bounded by
$\textbf{U} \succeq  \mathcal{F}\left(\textbf{J}^{\left(\ell\right)}\right)\succeq \textbf{L}\succ \textbf{0}$ for $\ell\geq 1$.

Based on the above properties of $\mathcal{F}\left(\cdot\right)$, we can establish the convergence of the information matrices.
\begin{theorem} \label{unique}
There exists a unique positive definite fixed point $\textbf{J}^{\ast}$ for the mapping $\mathcal F\left(\cdot\right)$.
\end{theorem}
\begin{proof}
The set $\left[\textbf{L}, \textbf{U}\right] $  is a compact set.
Further, according to Proposition \ref{P_FUN}, P \ref{P_FUN}.3,  for arbitrary
$\textbf{J}^{\left(0\right)}\succeq\textbf{0}$,  $\mathcal{F}$ maps $\left[\textbf{L}, \textbf{U}\right] $  into itself starting from
$\ell\geq 1$.
Next, we show that $\left[\textbf{L}, \textbf{U}\right]$ is a convex set.
 Suppose that $\textbf{X}$, $\textbf{Y}\in \left[\textbf{L},  \textbf{U}\right] $, and $0\leq t\leq 1$, then
 $t\textbf{X} - t\textbf{L} $ and $\left(1-t\right)\textbf{Y} - \left(1-t\right)\textbf{L} $ are positive
semidefinite (p.s.d.)  matrices.
Since the sum of two p.s.d. matrices is a
p.s.d. matrix,  $t\textbf{X} + \left(1-t\right)\textbf{Y}\succeq \textbf{L}$.
Likewise, it can be shown that
$t\textbf{X} + \left(1-t\right)\textbf{Y}  \preceq \textbf{U}$.
Thus, the continuous function $\mathcal F$ maps a compact convex subset of the Banach space of positive definite matrices into itself.
Therefore, the mapping $\mathcal F$ has a fixed point in $\left[\textbf{L}, \textbf{U}\right]$ according to  Brouwer's Fixed-Point Theorem \citep{FixedPoint}, and the fixed point is positive definite (p.d.).

Next, we prove the uniqueness of the fixed point.
Suppose that there exist two fixed points
$\textbf{J}^{\ast}\succ \textbf 0$ and $\widetilde{\textbf{J}}^{\ast}\succ \textbf 0$.
Since $\textbf{J}^{\ast}$ and $\widetilde{\textbf{J}}^{\ast}$
are p.d.,
their components $\textbf{J}_{f_n\to i}^{\ast}$
and $\widetilde{\textbf{J}}_{f_n\to i}^{\ast} $ are also p.d. matrices.
For the component blocks of
$\textbf{J}^{\ast} $ and $\widetilde{\textbf{J}}^{\ast} $, there are two possibilities: 1) $ \widetilde{\textbf{J}}_{f_n\to i}^{\ast}
- \textbf{J}_{f_n\to i}^{\ast}\succ \textbf{0} $ or
$ \widetilde{\textbf{J}}_{f_n\to i}^{\ast}
 -\textbf{J}_{f_n\to i}^{\ast} $ is indefinite for some $n, i \in \mathcal{V}$,
 and 2) $\widetilde{\textbf{J}}^{\ast}_{f_n\to i} -
 \textbf{J}^{\ast}_{f_n\to i}\preceq 0$ for all $n, i \in \mathcal{V}$.

For the first case, there must exist $\xi_{f_n,i}>1 $ such that
$\xi_{f_n,i}\textbf{J}_{f_n\to i}^{\ast}
-
\widetilde{\textbf{J}}_{f_n\to i}^{\ast}$ has one or more zero eigenvalues,
while all other eigenvalues are positive.  Pick the component matrix with the maximum $\xi_{f_n, i}$ among those falling
into this case, say  $\xi_{f_\varrho,\tau}$,  then, we can write
\begin{equation}\label{ine0}
\xi_{f_\varrho,\tau} \textbf{J}_{f_\varrho\to \tau}^{\ast}
-
  \widetilde{\textbf{J}}_{f_\varrho\to \tau}^{\ast}  \succeq \textbf{0},
\end{equation}
or in terms of the information matrices for the whole network
\begin{equation}\label{ine}
\xi_{f_\varrho,\tau}  \textbf{J}^{\ast} \succeq  \widetilde{\textbf{J}}^{\ast}\succ \textbf{0}, \quad \xi_{f_\varrho,\tau}>1.
\end{equation}
Applying $\mathcal F$ on both sides of (\ref{ine}), according to the monotonic property of $\mathcal F\left(\cdot\right)$ as shown in Proposition \ref{P_FUN}, P \ref{P_FUN}.1,
we have
\begin{equation}\label{27}
\mathcal F \left(\xi_{f_\varrho,\tau} \textbf{J}^{\ast} \right)
\succeq  \mathcal F\left(\widetilde{\textbf{J}}^{\ast}\right) = \widetilde{\textbf{J}}^{\ast},
\end{equation}
where the equality is due to $\widetilde{\textbf{J}}^{\ast} $ being a fixed point.
According to Proposition \ref{P_FUN}, P \ref{P_FUN}.2,
$\xi_{f_\varrho,\tau}\mathcal F \left( \textbf{J}^{\ast} \right)\succ\mathcal F \left(\xi_{f_\varrho,\tau} \textbf{J}^{\ast} \right)$.
Therefore, from (\ref{27}), we obtain
$\xi_{f_{\varrho,\tau}} \textbf{J}^{\ast}\succ\widetilde{\textbf{J}}^{\ast}$.
Consequently,
\begin{equation} \label{contr2}
\xi_{f_{\varrho,\tau}}
\textbf{J}_{f_{\varrho,\tau}}^{\ast}
\succ \widetilde{\textbf{J}}_{f_{\varrho,\tau}}^{\ast} .
\nonumber
\end{equation}
But this contradicts with
$\xi_{f_\varrho,\tau} \textbf{J}_{f_{\varrho,\tau}}^{\ast}
-
\widetilde{\textbf{J}}_{f_{\varrho,\tau}}^{\ast} $ having one or more zero eigenvalues as discussed before (\ref{ine0}).
Therefore,
we must have $\textbf{J}^{\ast}=\widetilde{\textbf{J}}^{\ast}$.

On the other hand, if we have case two, which is
$\widetilde{\textbf{J}}^{\ast}_{f_n\to i} - \textbf{J}^{\ast}_{f_n\to i} \preceq 0$
for all $n,i \in \mathcal{V}$, we can repeat the above derivation with the roles of $\widetilde{\textbf{J}}^{\ast}$ and $ \textbf{J}^{\ast}$ reversed, and we would again obtain $\textbf{J}^{\ast}=\widetilde{\textbf{J}}^{\ast}$.
Consequently, $\textbf{J}^{\ast}$ is unique.
\end{proof}

Lemma \ref{pdlemma} states that with arbitrary p.s.d. initial message information matrices, the message
information matrices will be kept as p.d. at every iteration.
On the other hand, Theorem \ref{unique} indicates that there exists a unique fixed point for the mapping $\mathcal F$.
Next, we will show that, with arbitrary initial value $\textbf{J}^{\left(0\right)}\succeq 0$, $\textbf{J}^{\left(\ell\right)}$ converges to
a unique  p.d. matrix.
\begin{theorem} \label{guarantee}
The matrix sequence
$\left\{\textbf{J}^{\left(\ell\right)}\right\}_{l=0,1,\ldots}$ defined by (\ref{CovFunc5}) converges to a unique positive definite matrix $\textbf{J}^{\ast}$
for any initial covariance matrix $\textbf{J}^{\left(0\right)}\succeq \mathbf 0$.
\end{theorem}
\begin{proof}
With arbitrary initial value $\textbf{J}^{\left(0\right)}\succeq \textbf 0$, following Proposition \ref{P_FUN}, P \ref{P_FUN}.3, we have
$\textbf{U} \succeq  \textbf{J}^{\left(1\right)}
\succeq \textbf{L}\succ \textbf{0}$.
On the other hand, according to Theorem \ref{unique},  (\ref{CovFunc5}) has a unique fixed point
$\textbf{J}^{\ast}\succ \textbf{0} $.
Notice that we can always choose a scalar
$\alpha > 1$ such that
\begin{equation}\label{inequality1}
 \alpha \textbf{J}^{\ast}
\succeq
\textbf{J}^{\left(1\right)}
\succeq
 \textbf{L}.
\end{equation}
Applying  $\mathcal F\left(\cdot\right)$ to (\ref{inequality1})
$\ell$ times, and using Proposition \ref{P_FUN}, P \ref{P_FUN}.1, we have
\begin{equation}\label{inequality2}
\mathcal F^{\ell}\left(\alpha \textbf{J}^{\ast}\right)
\succeq
\mathcal F^{\ell+1}\left(\textbf{J}^{\left(0\right)}\right)
\succeq
\mathcal F^{\ell}\left( \textbf{L}\right),
\end{equation}
where
$\mathcal F^{\ell}\left(\textbf{X}\right)$ denotes applying
$\mathcal F$ on $\textbf{X}$  $\ell$ times.

We start from the left inequality in (\ref{inequality2}).
According to
Proposition \ref{P_FUN}, P \ref{P_FUN}.2, $\alpha \textbf J^{\ast} \succ \mathcal F\left(\alpha \textbf J^{\ast}\right)$.
Applying $\mathcal F$ again gives $\mathcal F\left(\alpha \textbf J^{\ast}\right) \succ \mathcal F^2\left(\alpha \textbf J^{\ast}\right)$.
Applying $\mathcal F\left(\cdot\right)$ repeatedly, we can obtain $\mathcal F^2\left(\alpha \textbf{J}^{\ast}\right)
\succ \mathcal F^3\left(\alpha \textbf{J}^{\ast}\right)\succ \mathcal F^4\left(\alpha \textbf{J}^{\ast}\right)$, etc.
 Thus $\mathcal F^{\ell} \left(\alpha \textbf J^{\ast}\right)$ is a non-increasing sequence
with respect to the partial order induced by the cone of p.s.d. matrices
  as $\ell$ increases.
Furthermore, since $\mathcal F\left(\cdot\right)$ is bounded below by $\textbf L$, $\mathcal F^{\ell}\left(\alpha  \textbf J^{\ast}\right)$  converges.
Finally, since there exists only one fixed point for $\mathcal F\left(\cdot\right)$, $\lim_{l\to \infty} \mathcal F^{\ell} \left(\alpha \textbf J^{\ast}\right) = \textbf  J^{\ast}$.
On the other hand, for the right hand side of (\ref{inequality2}),
as $\mathcal F\left(\cdot\right)\succeq \textbf L$,
we have $\mathcal F\left(\textbf L\right)\succeq \textbf L$.
Applying $\mathcal F$ repeatedly gives successively $\mathcal F^2\left(\textbf L\right) \succeq \mathcal F\left(\textbf L\right)$,
$\mathcal F^3\left(\textbf L\right) \succeq \mathcal F^2\left(\textbf L\right)$, etc.
So, $\mathcal F^{\ell} \left(\textbf L\right)$ is an non-decreasing sequence (with respect to the partial order induced by the cone of p.s.d. matrices).
Since $\mathcal F\left(\cdot\right)$ is upper bounded by $\textbf U$, $\mathcal F^{\ell}\left(\textbf L\right)$ is a convergent sequence.
Again, due to the uniqueness of the fixed point, we have $ \lim_{l\to \infty} \mathcal F^{\ell}\left(\textbf L\right) =\textbf J^{\ast}$.
Finally, taking the limit with respect to $\ell$ on (\ref{inequality2}), we have
$\lim_{l\to \infty} \mathcal F^{\ell}\left(\textbf{J}^{\left(0\right)}\right) = \textbf{J}^{\ast},
$
for arbitrary initial $\textbf{J}^{\left(0\right)}\succeq \mathbf 0$.
\end{proof}

\begin{remark}
According to  Theorem \ref{guarantee}, the information matrix $\textbf{J}^{\left(\ell\right)}_{f_n\to i}$ converges  if all initial information matrices are
 p.s.d., i.e.,  $\textbf{J}^{\left(0\right)}_{f_n\to i} \succeq \textbf{0}$
 for all $i \in \mathcal V$ and $f_n \in \mathcal B\left(i\right)$.
{However, for the pairwise model, the messages are derived based on the classical Gaussian MRF based factorization (in the form of equations (\ref{MRF-local}) and  (\ref{MRF-pair})) of the joint distribution.
This  differs from the model considered in this paper, where the factor
$f_n$ follows equation  (\ref{jointpost}), which leads to  intrinsically different recursive equations.
More specifically,
for BP on the Gaussian MRF based factorization, the information matrix does not necessarily converge for all initial nonnegative values (for the scalar variable case) as shown in \citep{WalkSum1,minsum09}.}
\end{remark}
\begin{remark}
Due to the computation of $\textbf{J}^{\left(\ell\right)}_{f_n\to i}$  being independent of the local observations $\textbf{y}_n$,
as long as the network topology does not change, the converged value  $\textbf{J}^{\ast}_{f_n\to i}$ can be precomputed offline and stored at each agent, and there is no need to re-compute $\textbf{J}^{\ast}_{f_n\to i}$ even if $\textbf{y}_n$ varies.
\end{remark}

 Another fundamental question is how fast the convergence is, and this is the focus of the discussion below.
 Since the convergence of  a dynamic system is often studied with respect to the part metric \citep{PartBook},
in the following, we
start by introducing the part metric.

\begin{definition}\label{mydef}
Part (Birkhoff) Metric \citep{PartBook}:
For arbitrary {symmetric} matrices $\textbf{X}$ and $\textbf{Y}$ with the same dimension,
if there exists
$\alpha\geq 1$ such that $\alpha\textbf{X} \succeq \textbf{Y} \succeq \alpha^{-1} \textbf{X} $,
$\textbf{X} $ and $\textbf{Y}$
 are called the parts,
and
$ \mathrm{d} \left(\textbf{X}, \textbf{Y}\right)\triangleq
\inf \left\{\log \alpha: \alpha\textbf{X} \succeq \textbf{Y}\succeq \alpha^{-1} \textbf{X}, \alpha \geq 1\right\}$
defines a metric  called the part metric.
\end{definition}

As it is useful to have an estimate of the convergence rate of $\textbf{J}^{\left(\ell\right)}$ in terms of the more standard induced matrix norms, we further introduce the notion of monotone norms.
The norms $||\cdot||_2$
and $|| \cdot ||_F$ (Frobenus norm) are  monotone norms.
\begin{definition}\label{monotone}
	Monotone Norm \citep[2.2-10]{normbook}:
	A matrix norm $\|\cdot\|$ is  monotone if
	$$\textbf X\succeq \textbf 0, \textbf Y\succeq \textbf X \Rightarrow \|\textbf Y\|\geq \|\textbf X\|.$$
\end{definition}

Next, for arbitrary $\epsilon>0$, we will show that
$\left\{\textbf{J}^{\left(\ell\right)}\right\}_{l=1,..}$
approaches the $\epsilon$-neighborhood of the fixed point $\textbf J^{\star}$ double exponentially fast with respect to the monotone norm.
To this end, for a fixed $\epsilon>0$, define the set
\begin{equation}\label{partset}
\mathcal C =\left\{\textbf{J}^{\left(\ell\right)}|  \textbf{U} \succeq \textbf{J}^{\left(\ell\right)} \succeq \textbf{J}^{\ast}+ \epsilon \textbf{I}\right\}
\cup
\left\{\textbf{J}^{\left(\ell\right)}|
\textbf{J}^{\ast}- \epsilon \textbf{I} \succeq \textbf{J}^{\left(\ell\right)} \succeq \textbf{L}\right\}.
\end{equation}
\begin{theorem}\label{RateCov}
With the initial message information matrix set to be an arbitrary p.s.d. matrix, i.e., $\textbf{J}^{\left(0\right)}_{f_n\to i} \succeq \textbf{0}$,
the sequence $\left\{\textbf{J}^{\left(\ell\right)}\right\}_{l=0,1,\ldots}$ {  approaches an arbitrarily small neighborhood of the fixed positive definite matrix $\textbf{J}^{\ast}$  at a doubly exponential rate  with respect to any matrix norm.}
\end{theorem}
\begin{proof}
{ Fix $\epsilon>0$ and consider the set $\mathcal{C}$ defined in~(\ref{partset}). It suffices to show that the quantity $\|\textbf J^{(\ell)}-\textbf J^{\ast}\|$, where $\|\cdot\|$ is a monotone norm as defined in Definition~\ref{monotone}, decays double exponentially as long as $\textbf J^{(s)}\in\mathcal{C}$ for all $s\in\{0,1,\cdots\ell\}$. To this end, for
{$\textbf{J}^{\left(\ell\right)} \in \mathcal C$, and $\textbf{J}^{\ast} \not\in \mathcal C$ (necessarily),
according to Definition \ref{mydef},
we have
$ \mathrm{d} \left(\textbf{J}^{\left(\ell\right)}, \textbf{J}^{\ast}\right)\triangleq
\inf \left\{\log\alpha: \alpha\textbf{J}^{\left(\ell\right)} \succeq \textbf{J}^{\ast}\succeq \alpha^{-1} \textbf{J}^{\left(\ell\right)}\right\}$.}}
Since $\mathrm{d} \left(\textbf{J}^{\left(\ell\right)}, \textbf{J}^{\ast}\right)$ is the smallest number satisfying $\alpha \textbf{J}^{\left(\ell\right)} \succeq \textbf{J}^{\ast} \succeq \alpha^{-1} \textbf{J}^{\left(\ell\right)}$, this is equivalent to
\begin{equation}\label{33}
\exp\left\{\mathrm{d} \left(\textbf{J}^{\left(\ell\right)}, \textbf{J}^{\ast}\right)\right\}\textbf{J}^{\left(\ell\right)}
\succeq
\textbf{J}^{\ast}
\succeq
\exp\left\{-\mathrm{d} \left(\textbf{J}^{\left(\ell\right)}, \textbf{J}^{\ast}\right)\right\}
\textbf{J}^{\left(\ell\right)}.
\end{equation}
Applying Proposition \ref{P_FUN}, P \ref{P_FUN}.1 to (\ref{33}), we have
\begin{equation}
\mathcal F\left(\exp\left\{\mathrm{d} \left(\textbf{J}^{\left(\ell\right)}, \textbf{J}^{\ast}\right)\right\}\textbf{J}^{\left(\ell\right)}\right)
\succ
\mathcal F\left(\textbf{J}^{\ast}\right)
\succ
\mathcal F\left(\exp\left\{-\mathrm{d} \left(\textbf{J}^{\left(\ell\right)}, \textbf{J}^{\ast}\right)\right\}
\textbf{J}^{\left(\ell\right)}\right).\nonumber
\end{equation}
Then applying Proposition~\ref{P_FUN}, P~\ref{P_FUN}.2 and
considering  that   $\exp\left\{\mathrm{d} \left(\textbf{J}^{\left(\ell\right)}, \textbf{J}^{\ast}\right)\right\}>1$ and
$\exp\left\{-\mathrm{d} \left(\textbf{J}^{\left(\ell\right)}, \textbf{J}^{\ast}\right)\right\}<1$, we obtain
$$\exp\left\{\mathrm{d} \left(\textbf{J}^{\left(\ell\right)}, \textbf{J}^{\ast}\right)\right\}
\mathcal F\left(\textbf{J}^{\left(\ell\right)}\right)
\succeq
\mathcal F\left(\textbf{J}^{\ast}\right)
\succeq
\exp\left\{-\mathrm{d} \left(\textbf{J}^{\left(\ell\right)}, \textbf{J}^{\ast}\right)\right\}
\mathcal F\left(\textbf{J}^{\left(\ell\right)}\right).
$$
Notice that, for arbitrary p.d. matrices $\textbf{X}$ and $\textbf{Y}$, if $\textbf{X}-k\textbf{Y}\succ \textbf{0}$, then, by definition, we have $\textbf{x}^T\textbf{X}\textbf{x}
-k\textbf{x}^T\textbf{Y}\textbf{x}
> {0}$ for arbitrary $\textbf x\neq \textbf 0$.
Then, there must exist $o>0$ that is small enough such that
$\textbf{x}^T\textbf{X}\textbf{x}
-\left(k+o\right)
\textbf{x}^T\textbf{Y}\textbf{x}
> {0}$
or equivalently
$\textbf{X}
\succ \left(k+o\right)
\textbf{Y}$.
Thus, as $\exp{\left(\cdot\right)}$ is a continuous function, there must exist some $\triangle\mathrm{d}>0$ such that
\begin{equation}\label{35}
\exp\left\{-\triangle\mathrm{d}+\mathrm{d} \left(\textbf{J}^{\left(\ell\right)}, \textbf{J}^{\ast}\right)\right\}
\mathcal F\left(\textbf{J}^{\left(\ell\right)}\right)
\succ
\mathcal F\left(\textbf{J}^{\ast}\right)
\succ
\exp\left\{\triangle\mathrm{d}-
\mathrm{d} \left(\textbf{J}^{\left(\ell\right)}, \textbf{J}^{\ast}\right)\right\}
\mathcal F\left(\textbf{J}^{\left(\ell\right)}\right).
\end{equation}
Now, using  the definition of the part metric, (\ref{35}) is equivalent to
\begin{equation}
-\triangle\mathrm{d}+\mathrm{d}
\left(\textbf{J}^{\left(\ell\right)}, \textbf{J}^{\ast}\right)
\geq
\mathrm{d} \left(\mathcal F\left(\textbf{J}^{\left(\ell\right)}\right), \mathcal F\left(\textbf{J}^{\ast}\right)\right).\nonumber
\end{equation}
Hence, we obtain
$\mathrm{d} \left(\mathcal F\left(\textbf{J}^{\left(\ell\right)}\right), \mathcal F\left(\textbf{J}^{\ast}\right)\right)
<
\mathrm{d}
\left(\textbf{J}^{\left(\ell\right)}, \textbf{J}^{\ast}\right)$.
Since this result holds for any $\textbf{J}^{\left(\ell\right)} \in \mathcal C$, we also have
$\mathrm{d} \left(\mathcal{F}\left(\textbf{J}^{\left(\ell\right)}\right), \mathcal{F}\left(\textbf{J}^{\ast}\right)\right) <
c
\mathrm{d} \left(\textbf{J}^{\left(\ell\right)},\textbf{J}^{\ast}\right) $,
where $c=\sup_{\textbf{J}^{(\ell)}\in \mathcal{C}} \frac{\mathrm{d}\left(\mathcal{F}\left(\textbf{J}^{(\ell)}\right) , \mathcal{F}\left(\textbf{J}^{\ast}\right)\right) }{\mathrm{d} \left(\textbf{J}^{(\ell)},\textbf{J}^{\ast}\right)}<1$.
Since $\textbf{J}^{\ell + 1} = \mathcal{F}\left(\textbf{J}^{\ell}\right)$ and $\textbf{J}^{\ast} = \mathcal{F}\left(\textbf{J}^{\ast}\right)$,
we have
\begin{equation}\label{rate}
\mathrm{d} \left(\textbf{J}^{\left(\ell\right)}, \textbf{J}^{\ast}\right) <
c^{\ell}
\mathrm{d} \left(\textbf{J}^{\left(0\right)}, \textbf{J}^{\ast}\right).
\end{equation}

According to  \citep[Lemma 2.3]{matrixnorm},
the convergence rate of $||\textbf{J}^{\left(\ell\right)}- \textbf{J}^{\ast}||$ can be determined by that of $\mathrm{d} \left(\textbf{J}^{\left(\ell\right)}, \textbf{J}^{\ast}\right)$. More specifically,
\begin{equation}\label{rateineq}
||\textbf{J}^{\left(\ell\right)} - \textbf{J}^{\ast}||
\leq
\left(2\exp\left\{\mathrm{d} \left(\textbf{J}^{\left(\ell\right)}, \textbf{J}^{\ast}\right)\right\} - \exp\left\{-\mathrm{d} \left(\textbf{J}^{\left(\ell\right)}, \textbf{J}^{\ast}\right)\right\} - 1\right)
\min\left\{||\textbf{J}^{\left(\ell\right)}||, ||\textbf{J}^{\ast}||\right\},
\end{equation}
where $||\cdot||$ is a monotone norm defined on the p.s.d. cone:

{As we show in Proposition \ref{P_FUN}, P \ref{P_FUN}.3 that $\textbf{J}^{\left(\ell\right)}$ is bounded, then
$||\textbf{J}^{\left(\ell\right)}||$ and
$||\textbf{J}^{\ast}||$ must be finite.
Let  $\zeta$ be the largest value of
$\min\left\{||\textbf{J}^{\left(\ell\right)}||, ||\textbf{J}^{\ast}||\right\}$ for all $\{\textbf{J}^{\left(\ell\right)}\}$  with $\ell\geq 0$, then $\zeta>0$.
According to (\ref{rate}) and  (\ref{rateineq}), we have that
\begin{equation}\label{rateineq2}
||\textbf{J}^{\left(\ell\right)} - \textbf{J}^{\ast}||
<
\zeta
\left(2\exp\left\{c^{\ell}d_0\right\} - \exp\left\{- c^{\ell}d_0\right\} - 1\right),
\end{equation}
with $0<c<1$ and  $d_0=\mathrm{d} \left(\textbf{J}^{\left(0\right)}, \textbf{J}^{\ast}\right)$, which is a constant.
{ 
The above inequality is equivalent to
\begin{equation}\label{rateineq3}
||\textbf{J}^{\left(\ell\right)} - \textbf{J}^{\ast}||
<
\zeta
\left(3\exp\left\{c^{\ell}d_0\right\} -\exp\left\{c^{\ell}d_0\right\} - \exp\left\{- c^{\ell}d_0\right\} - 1\right).
\end{equation}
Since both $\exp\left\{c^{\ell}d_0\right\}$ and $\exp\left\{- c^{\ell}d_0\right\}$ are positive and $\exp\left\{c^{\ell}d_0\right\}\exp\left\{- c^{\ell}d_0\right\} = 1$, according to the arithmetic-geometric mean inequality,
we have
$\exp\left\{c^{\ell}d_0\right\}+ \exp\left\{- c^{\ell}d_0\right\}\geq 2\left(\exp\left\{c^{\ell}d_0\right\}\exp\left\{- c^{\ell}d_0\right\}\right)^{1/2}=2$.
Then, the right-hand side of (\ref{rateineq3}) is further amplified, and we obtain
\begin{equation}\label{rateineq4}
||\textbf{J}^{\left(\ell\right)} - \textbf{J}^{\ast}||
<
\zeta
\left(3\exp\left\{c^{\ell}d_0\right\} - 3\right) = 3\zeta
\left(\exp\left\{c^{\ell}d_0\right\} - 1\right).\nonumber
\end{equation}
Therefore, the sequence $\left\{\textbf{J}^{\left(\ell\right)}\right\}_{l=0,1,\ldots}$  approaches the $\epsilon$-neighborhood (and hence any arbitrarily small neighborhood) of the fixed positive definite matrix $\textbf{J}^{\ast}$  at a doubly exponential rate  with respect to any matrix norm.}}
\end{proof}
The physical meaning of Theorem \ref{RateCov} is that the distance between $\textbf{J}^{\left(\ell\right)}$ and $\textbf{J}^{\ast}$ decreases doubly exponentially fast before $\textbf{J}^{\left(\ell\right)}$ enters $\textbf{J}^{\ast}$'s neighborhood, which can be chosen to be arbitrarily small.
Next, we study  how to choose the initial value
$\textbf{J}^{\left(0\right)}$ so that $\textbf{J}^{\left(\ell\right)}$ converges faster.

\begin{theorem}\label{Coro1}
With
$\textbf{0} \preceq \textbf{J}^{\left(0\right)}\preceq \textbf{L}$,
$\textbf{J}^{\left(\ell\right)}$ is a monotonic increasing sequence, and
$\textbf{J}^{\left(\ell\right)}$ converges most rapidly with
$\textbf{J}^{\left(0\right)}= \textbf{L}$.
Moreover, with
$\textbf{J}^{\left(0\right)}\succeq \textbf{U}$,
$\textbf{J}^{\left(\ell\right)}$ is a monotonic decreasing  sequence, and
$\textbf{J}^{\left(\ell\right)}$ converges most rapidly
with
 $\textbf{J}^{\left(0\right)}=
\textbf{U}$.
\end{theorem}
\begin{proof}
Following Proposition \ref{P_FUN}, P \ref{P_FUN}.3, it can be  verified that for $\textbf{0}\preceq \textbf{J}^{\left(0\right)}\preceq \textbf{L}$, we have
$\textbf{J}^{\left(1\right)}\succeq \textbf{J}^{\left(0\right)}$.
Then, according to Proposition \ref{P_FUN}, P \ref{P_FUN}.1,
and by induction, this relationship can be extended to
${\textbf{J}}^{\left(\ell\right)}\succeq \ldots {\textbf{J}}^{\left(1\right)}
\succeq {\textbf{J}}^{\left(0\right)}$,
which states that ${\textbf{J}}^{\left(\ell\right)}$ is a monotonic increasing sequence.
Now,
suppose that there are two sequences $\textbf{J}^{\left(\ell\right)}$ and
$\widetilde{\textbf{J}}^{\left(\ell\right)}$ that are
started with different initial values
 $\textbf{0}\preceq \textbf{J}^{\left(0\right)}\prec \textbf{L} $
and $\textbf{0}\preceq
\widetilde{\textbf{J}}^{\left(0\right)} \prec \textbf{L} $, respectively.
Then these two sequences are monotonically increasing and  bounded by $\textbf{J}^{\ast}$.
To prove that $\textbf{J}^{\left(0\right)}=\textbf{L}$ leads to the fastest convergence,
it is sufficient to prove that
$\textbf{J}^{\left(\ell\right)}  \succ
\widetilde{\textbf{J}}^{\left(\ell\right)}$ for $\ell=0, 1 \ldots$.
First, note that
$\textbf{J}^{\left(0\right)}  \succ
\widetilde{\textbf{J}}^{\left(0\right)}$.
Assume $\textbf{J}^{\left(n\right)}  \succ
\widetilde{\textbf{J}}^{\left(n\right)}$ for some $n\geq 0$.
According to Proposition \ref{P_FUN}, P \ref{P_FUN}.1,
we have $\mathcal{F}\left(\textbf{J}^{\left(n\right)}\right) \succeq \mathcal{F}\left(\widetilde{\textbf{J}}^{\left(n\right)}\right)$,
or equivalently
$\textbf{J}^{\left(n+1\right)}  \succeq
\widetilde{\textbf{J}}^{\left(n+1\right)}$.
Therefore, by induction,
 we have proven that,
with $\textbf{J}^{\left(0\right)}= \textbf{L}$,
$\textbf{J}^{\left(\ell\right)}$ converges more rapidly than  with any other initial value
$\textbf{0}\preceq \textbf{J}^{\left(0\right)}\prec \textbf{L} $.

With similar logic,  we can show that,
 with
$\textbf{J}^{\left(0\right)}\succeq \textbf{U}$,
$\textbf{J}^{\left(\ell\right)}$ is a monotonic decreasing  sequence; and, with
$\textbf{J}^{\left(0\right)}=
\textbf{U}$,
$\textbf{J}^{\left(\ell\right)}$ converges more rapidly than that with any other initial value
$\textbf{J}^{\left(0\right)}\succ \textbf{U} $.
\end{proof}

Notice that it is a common practice in the Gaussian BP literature that the initial information matrix (or inverse variance  for the scalar case) is set  to be $\textbf{0}$, i.e.,
$\textbf{J}^{\left(0\right)}_{f_n\to i} =\mathbf {0}$
\citep{DiagnalDominant, WalkSum1}.
Theorem \ref{Coro1} reveals that there is a better
choice  to guarantee faster convergence.

\subsection{Convergence of Message Mean Vector}\label{B}

According to Theorems \ref{guarantee} and \ref{RateCov},
as long as we choose $\textbf{J}_{f_k\to j}^{\left(0\right)}\succeq \textbf{0}$
for all $j\in \mathcal V$ and $f_k\in \mathcal B\left(j\right)$,
the distance between $\textbf{J}_{f_k\to j}^{\left(\ell\right)}$ and $\textbf{J}_{f_k\to j}^{\ast}$ decreases doubly exponentially fast before $\textbf{J}_{f_k\to j}^{\left(\ell\right)}$ enters  $\textbf{J}_{f_k\to j}^{\ast}$'s neighborhood, which can be chosen to be arbitrarily small.
Furthermore, according to (\ref{v2fV}),
$
\left[\textbf{J}^{\left(\ell\right)}_{j \to f_n}\right]^{-1}$ also converges to a p.d. matrix once
$\textbf{J}_{f_k\to j}^{\left(\ell\right)}$ converges, and the converged value for $
\left[\textbf{J}^{\left(\ell\right)}_{j \to f_n}\right]^{-1}$ is denoted by
$
\left[\textbf{J}^{\ast}_{j \to f_n}\right]^{-1}$.
Then for arbitrary initial value
$\textbf{v}^{\left(0\right)}_{f_k\to j}$, the evolution of
$\textbf{v}^{\left(\ell\right)}_{j\to f_n}$ in
(\ref{v2fm}) can be written in terms of the converged message information matrices, which is
\begin{equation}\label{v2fm36}
\textbf{v}^{\left(\ell\right)}_{j\to f_n}=
\left[\textbf{J}^{\ast}_{j\to f_n}\right]^{-1}
\sum_{f_k\in\B\left(j\right)\setminus f_n}
\textbf{J}_{f_k\to j}^{\ast}
\textbf{v}^{\left(\ell-1\right)}_{f_k\to j}.
\end{equation}
Using (\ref{f2vmm}), and
replacing indices $j$, $i$, $n$ with $z$, $j$, $k$ respectively,
$\textbf{v}^{\left(\ell-1\right)}_{f_k\to j}$  is given by
\begin{equation}\label{f2vmm37}
\begin{split}
\textbf{v}^{\left(\ell-1\right)}_{f_k\to j}
=
[\textbf{J}_{f_k\to j}^{\ast}]^{-1}
\textbf{A}_{k,j}^T
\Bigg[\underbrace{\textbf{R}_k
+
\sum_{z\in\B\left(f_k\right)\setminus j} \textbf{A}_{k,z}
\left[\textbf{J}^{\ast}_{z\to f_k}\right]^{-1}\textbf{A}_{k,z}^T
}_{\triangleq\textbf{M}_{k,j}} \Bigg]^{-1}
\left(\textbf{y}_k-\sum_{z\in\B\left(f_k\right)\setminus j} \textbf{A}_{k,z}
\textbf{v}^{\left(\ell-1\right)}_{z\to f_k}\right).
\end{split}
\end{equation}
Putting (\ref{f2vmm37}) into
 (\ref{v2fm36}), we have
\begin{equation}\label{v2fm3}
\textbf{v}^{\left(\ell\right)}_{j\to f_n}=
\textbf{b}_{j \to f_n}-
\sum_{f_k\in\B\left(j\right)\setminus f_n}
\sum_{z\in\B\left(f_k\right)\setminus j}
[\textbf{J}^{\ast}_{j\to f_n}]^{-1}
\textbf{A}_{k,j}^T\textbf{M}_{k,j}^{-1}
\textbf{A}_{k,z}
\textbf{v}^{\left(\ell-1\right)}_{z\to f_k},
\end{equation}
where $\textbf{b}_{j\to f_n}=
[\textbf{J}^{\ast}_{j\to f_n}]^{-1}
\sum_{f_k\in\B\left(j\right)\setminus f_n}
\textbf{A}_{k,j}^T\textbf{M}_{k,j}^{-1}
\textbf{y}_k $.
The above equation can be further written in compact form as
\begin{equation}\label{v2fm4}
\textbf{v}^{\left(\ell\right)}_{j\to f_n}=
\textbf{b}_{j \to f_n} -
\textbf{Q}_{j \to f_n}
\textbf{v}^{\left(\ell-1\right)},\nonumber
\end{equation}
with the column vector $\textbf{v}^{\left(\ell-1\right)}$
 containing  $\textbf{v}^{\left(\ell-1\right)}_{z\to f_k}$ for all $z\in \mathcal V$ and $f_k\in \mathcal B\left(z\right)$ as subvector
with ascending index first on $z$ and then on $k$.
The matrix   $\textbf{Q}_{j \to f_n}$
is a row block matrix with component block
$[\textbf{J}^{\ast}_{j\to f_n}]^{-1}
\textbf{A}_{k,j}^T\textbf{M}_{k,j}^{-1}
\textbf{A}_{k,z}$
if $f_k\in\B\left(j\right)\setminus f_n$ and ${z\in\B\left(f_k\right)\setminus j}$, and $\textbf{0}$ otherwise.
Let $\textbf{Q}$ be the block matrix that stacks  $\textbf{Q}_{j \to f_n}$ with the order first on $j$ and then on $ n$,
and  $\textbf{b}$ be the vector  containing
$\textbf{b}_{j \to f_n}$ with the same stacking order as $\textbf{Q}_{j \to f_n}$. We have
\begin{equation}\label{meanvectorupdate}
\textbf{v}^{\left(\ell\right)} =
 -\textbf{Q} \textbf{v}^{\left(\ell-1\right)} + \textbf{b}, \quad \ell\geq 1,2,\ldots.
\end{equation}
It is known that for arbitrary initial value $\textbf{v}^{\left(0\right)}$, $\textbf{v}^{\left(\ell\right)}$ converges
if and only if the spectral radius $\rho\left(\textbf{Q}\right)<1$ \citep[pp. 280]{Demmel}.
Since the elements of $\textbf{v}^{\left(0\right)}$, i.e., $\textbf v_{j\to f_n}^{(0)}$, depends on $\textbf v_{f_k\to j}^{(0)}$, we can choose arbitrary $\textbf v_{f_k\to j}^{(0)}$.
Furthermore, as $\textbf{v}^{\left(\ell\right)}$ depends on the convergence of  $\textbf{J}^{\left(\ell\right)}$, we have the following result.
\begin{theorem} \label{meanvector}
The vector sequence
$\left\{\textbf{v}^{\left(\ell\right)}\right\}_{l=1, 2,\ldots}$ defined by (\ref{meanvectorupdate}) converges to a unique value
under any initial value $\left\{
\textbf v_{f_k\to j}^{(0)}
\right\}_{k\in \mathcal V, j\in \B(f_k)}$ and initial message information matrix $\textbf{J}^{\left(0\right)}\succeq \mathbf 0$ if and only if $\rho\left(\mathbf {Q}\right)<1$.
\end{theorem}

{{
The row block matrix $\textbf Q_j$, a row block of $\textbf Q$, contains only block entries $\textbf 0$ and $\textbf Q_{j\to f_n}$.
When the observation model (1)  reduces  to the pairwise model, where only two unknown variables are involved in each local observation, it can be shown that $\textbf Q_j$ and $\textbf Q_i$ are orthogonal if $i\neq j$. A distributed convergence condition is obtained utilizing this orthogonal property  in \citet{DconveCon}.
However, for the more general case studied in this paper,  properties of $\textbf Q_j$ and $\textbf Q$ need to be further  exploited to show when $\rho(\textbf Q)<1$ is satisfied. }

In the sequel, we will show that
$\rho \left(\textbf{Q}\right) <1$ is satisfied for a single loop factor graph with multiple chains/trees (an example is shown in Fig. 2), thus Gaussian BP converges  in such a topology.
Although \citet{Weiss2000}
shows the convergence of Gaussian BP on the MRF with a single loop, the analysis  cannot be applied here since the local observations model (\ref{linear})  is different from the pairwise model in \citep{Weiss2000}.

\begin{theorem}\label{ref}
For any factor graph that is
 the union of a single loop and a forest,
with arbitrary positive semi-definite initial information matrix, i.e., $\textbf{J}^{\left(0\right)}_{f_n\rightarrow i}\succeq \textbf{0}$ for all $i\in \mathcal V$ and $f_n\in \mathcal B\left(i\right)$, the message information matrix $\textbf{J}^{\left(\ell\right)}_{f_n\rightarrow i}$ and mean vector
$\textbf{v}^{\left(\ell\right)}_{i\rightarrow f_n}$ is guaranteed to
  converge to their corresponding unique points.
\end{theorem}
\begin{proof}
{In this proof,  Fig.~2 is being used as a reference throughout.}
For a single loop factor graph with chains/trees as shown in Fig.~2 (a), there are two kinds of nodes.
One  is the factors/variables in the loop, and they are denoted by
$f_n/\textbf{x}_j$.
The other  is the factors/variables on the chains/trees but outside the loop, denoted as
$\widetilde{f}_k/\widetilde{\textbf{z}}_i$.
Then message from a variable node to a neighboring factor node on the graph can be categorized into three groups:

\noindent 1) message on a tree/chain passing towards the   loop, e.g.,
$m_{\widetilde{z}\to {f}_k}^{\ast}\left(\widetilde{\textbf{x}}_z\right)$
and
$m_{\widetilde{s}\to \widetilde{f}_k}^{\ast}\left(\widetilde{\textbf{x}}_s\right)$
;

\noindent 2) message on a tree/chain passing away from the loop, e.g.,
$m_{j\to \widetilde{f}_k}^{\left(\ell\right)}\left({\textbf{x}}_j\right)$,
$m_{\widetilde{s}\to \widetilde{f}_s}^{\left(\ell\right)}\left({\textbf{x}}_s\right)$ and
$m_{\widetilde{z}\to \widetilde{f}_z}^{\left(\ell\right)}\left({\textbf{x}}_z\right)$;

\noindent 3) message in the loop, e.g.,
$m_{j\to f_n}^{\left(\ell\right)}\left(\textbf{x}_j\right)$,
$m_{z\to {f}_k}^{\left(\ell\right)}\left({\textbf{x}}_z\right)$
and
$m_{i\to {f}_n}^{\left(\ell\right)}\left({\textbf{x}}_i\right)$.

\begin{figure}
  \centering
\mbox{\subfigure[]{\epsfig{figure=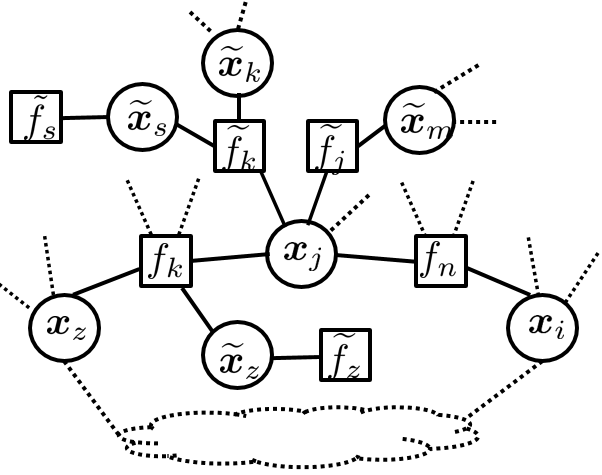,width=2.8in}}\label{SingleLoopA} }\\
\mbox{ \subfigure[]{\epsfig{figure=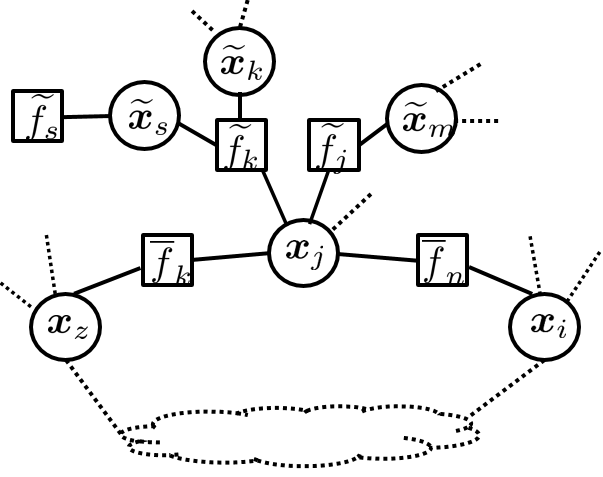,width=2.8in}}\label{SingleLoopB} }\\
\caption{(a) An example of factor graph with a single loop and chains/trees topology, where the dashed line indicates possible chains/trees;
(b) The equivalent factor graph of Fig 2 (a) with new factor functions that do not have neighboring variable nodes except those in the loop.}
\end{figure}

According to (\ref{BPv2f1}), computation of  the messages in the first group does not depend on messages in the loop and is thus convergence guaranteed.
Therefore, the message iteration number is replaced with a $\ast$ to denote the converged message.
Also,  from the definition of message computation in
(\ref{BPv2f1}),
 if messages in the third group converge, the second group messages  should also converge.
Therefore, we next focus on showing the convergence of messages in the third group.

For a factor node $f_k$ in the loop with $\textbf x_z$ and $\textbf x_j$ being its two neighboring variable nodes in the loop and $\widetilde{\textbf x}_z$ being its neighboring variable node outside the loop, according to the definition of message computation in (\ref{BPf2v1}), we have
\begin{equation}\label{BPf2v41}
\begin{split}
m^{\left(\ell\right)}_{f_k \to j}\left(\textbf{x}_j\right)
&=  \int \int
f_k \times
m^{ \left(\ell\right)}_{z \to f_k}\left(\textbf x_z\right)
\prod_{\widetilde{z}\in\B\left(f_k\right)\setminus j} m^{ \ast}_{\widetilde{z} \to f_k}\left(
\widetilde{\textbf{x}}_z\right)
\,\mathrm{d}\left\{
\widetilde{\textbf{x}}_z\right\}_{\widetilde{z}\in\B\left(f_k\right)\setminus j}
\,\mathrm{d}\textbf{x}_z,\\
&
=
\int
m^{ \left(\ell\right)}_{z \to f_k}\left(\textbf x_z\right)
\left[
\int
f_k \times
\prod_{\widetilde{z}\in\B\left(f_k\right)\setminus j} m^{ \ast}_{\widetilde{z} \to f_k}\left(
\widetilde{\textbf{x}}_z\right)
\,\mathrm{d}\left\{
\widetilde{\textbf{x}}_z\right\}_{\widetilde{z}\in\B\left(f_k\right)\setminus j}\right]
\,\mathrm{d}\textbf{x}_z.
\end{split}
\end{equation}
As shown in Lemma \ref{pdlemma},
$m^{ \ast}_{\widetilde{z} \to f_k}\left(
\widetilde{\textbf{x}}_z\right)$ must be in  Gaussian function form, which is denoted by
$m^{ \ast}_{\widetilde{z} \to f_k}\left(
\widetilde{\textbf{x}}_z\right)
\propto
\mathcal{N}\left(\widetilde{\textbf{x}}_z|
\textbf{v}_{\widetilde{z} \to f_k}^{\ast}, \left[\textbf{J}_{\widetilde{z} \to f_k}^{\ast}\right]^{-1}\right)$.
Besides, from (\ref{linear}) we obtain
\begin{displaymath}f_k =\mathcal N\left(\textbf{y}_k|
\textbf{A}_{k,z}{\textbf{x}}_z
+
\textbf{A}_{k,j}{\textbf{x}}_j
+
\sum_{\widetilde{z}\in  \mathcal{B}\left(f_k\right)}
\textbf{A}_{k,\widetilde{z}}\widetilde{\textbf{x}}_z, \textbf{R}_k\right).\end{displaymath}
It can be shown that
the inner integration in the second line of (\ref{BPf2v41}) is given by
\begin{displaymath}
\mathcal N\left(\overline{\textbf y}_k|
\overline{\textbf{A}}_{k,z} \textbf{x}_z\\+
\overline{\textbf{A}}_{k,j} \textbf{x}_j, \overline{\textbf{R}}_k\right)
\triangleq\overline{f}_k,\end{displaymath}
where
the overbar is used  to denote the new constant matrix or vector.
Then (\ref{BPf2v41}) can be written as
\begin{equation}\label{newdef}
\begin{split}
m^{\left(\ell\right)}_{f_k \to j}\left(\textbf{x}_j\right)
=
\int
\overline{f}_k\times
m^{ \left(\ell\right)}_{z \to f_k}\left(\textbf x_z\right)
\,\mathrm{d}\textbf{x}_z.
\end{split}
\end{equation}
Comparing  (\ref{newdef}) with (\ref{BPf2v1}), we obtain
$m^{\left(\ell\right)}_{f_k \to j}\left(\textbf{x}_j\right)$  as if
$m^{\left(\ell\right)}_{\overline{f}_k \to j}\left(\textbf{x}_j\right)$
is being passed to a factor node
$\overline{f}_k$. Therefore, a
factor graph with a single loop and multiple trees/chains is equivalent
to  a single loop factor graph in which each factor node has no neighboring variable node outside the loop.
As a result,
the example of Fig. 2 (a)   is equivalent to Fig. 2 (b).
In the following, we focus on this equivalent topology for the convergence analysis.

Note that, for arbitrary variable node $j$ in the loop, there are two neighboring
factor nodes in the loop.
Further, using the notation for the equivalent topology,
(\ref{v2fm3}) is reduced to
\begin{equation}\label{46}
\begin{split}
\textbf{v}^{\left(\ell\right)}_{j\to \overline{f}_n}
=&
-
\left[\textbf{J}^{\ast}_{j\to \overline{f}_n}\right]^{-1}
\overline{\textbf{A}}_{k,j}^T
\textbf{T}_{k,j}^{-1}
\overline{\textbf{A}}_{k,z}
\textbf{v}^{\left(\ell-1\right)}_{z\to \overline{f}_k}\\
&+
\underbrace{\overline{\textbf{b}}_{j \to \overline{f}_n}-
\sum_{\widetilde{f}_k\in\B\left(j\right)\setminus f_n}
\sum_{\widetilde{s}\in\B\left(\widetilde{f}_k\right)\setminus j}
\left[\textbf{J}^{\ast}_{j\to \overline{f}_n}\right]^{-1}
\overline{\textbf{A}}_{k,j}^T\overline{\textbf{M}}_{k,j}^{-1}
\overline{\textbf{A}}_{k,\widetilde{s}}
\textbf{v}^{\ast}_{\widetilde{s}\to \widetilde{f}_k}}_{\triangleq \textbf{c}_{j\to \overline{f}_n}},
\end{split}
\end{equation}
where $\textbf{v}^{\ast}_{\widetilde{s}\to \widetilde{f}_k}$
 is the converged mean vector on the chain/tree;
 $$\overline{\textbf{b}}_{j\to \overline{f}_n}=
\left[\textbf{J}^{\ast}_{j\to \overline{f}_n}\right]^{-1}
\sum_{\overline{f}_k\in\B\left(j\right)\setminus \overline{f}_n}
\overline{\textbf{A}}_{k,j}^T
\overline{\textbf{M}}_{k,j}^{-1}
\overline{\textbf{y}}_k $$ with
$\overline{\textbf{M}}_{k,j} = \overline{\textbf{R}}_k
+
\sum_{\widetilde{s}\in\B\left(\widetilde{f}_k\right)\setminus j} \overline{\textbf{A}}_{k,\widetilde{s}}
\left[\textbf{J}^{\ast}_{\widetilde{s}\to \widetilde{f}_k}\right]^{-1}\overline{\textbf{A}}_{k,\widetilde{s}}^T$, and
 \begin{equation}\label{M}
 \textbf{T}_{k,j}
 =
 \overline{\textbf{R}}_k
+
\overline{\textbf{A}}_{k,z}
\left[\textbf{J}^{\ast}_{z\to \overline{f}_k}\right]^{-1}\overline{\textbf{A}}_{k,z}^T,
 \end{equation}
 with $\textbf{x}_z$ and $\overline{f}_k$
 in the loop where
$\overline{f}_k \in \mathcal B\left(j\right)\setminus \overline{f}_n$ and $\textbf{x}_z \in \mathcal B\left(\overline{f}_k\right)\setminus j
$.
By multiplying $\left[\textbf J^{\ast}_{j\to  \overline{f}_{n}}\right]^{1/2}$ on both sides of (\ref{46}),
and defining $\bm \beta_{j\to  \overline{f}_{n}}^{\left(\ell\right)}=
\left[\textbf J^{\ast}_{j\to  \overline{f}_{n}}\right]^{1/2}
\textbf v^{\left(\ell\right)}_{j\to  \overline{f}_{n}}$,
 we have
\begin{equation}\label{47}
\bm \beta^{\left(\ell\right)}_{j\to  \overline{f}_{n}}=
-
\left[\textbf J^{\ast}_{j\to  \overline{f}_{n}}\right]^{-1/2}
\overline{\textbf{A}}_{k,j}^T
\textbf{T}_{k,j}^{-1}
\overline{\textbf{A}}_{k,z}
\left[\textbf{J}^{\ast}_{z\to  \overline{f}_{k}}\right]^{-1/2}
\bm \beta^{\left(\ell-1\right)}_{z\to \overline{f}_{k}}
+
\left[\textbf J^{\ast}_{j\to  \overline{f}_{n}}\right]^{1/2}\textbf{c}_{j\to \overline{f}_{n}},
\end{equation}
Let $\bm{\beta}^{\left(\ell-1\right)}$
 contain $\bm{\beta}^{\left(\ell-1\right)}_{z\to \overline{f}_{k}}$ for all
$\textbf{x}_z$ with $z\in \mathcal B\left(\overline{f}_k\right)$ and
$\overline{f}_k$ being in the loop, and the index is arranged first on $k$ and then on $z$.
Then,
the above equation is written in a compact form as
\begin{equation}\label{linear-vector}
\bm \beta^{\left(\ell\right)}_{j\to  \overline{f}_{n}}
= -\textbf{Q}_{j\to  \overline{f}_{n}} \bm \beta^{\left(\ell-1\right)} +
\left[\textbf J^{\ast}_{j\to  \overline{f}_{n}}\right]^{1/2}
\textbf{c}_{j\to  \overline{f}_{n}},
\end{equation}
where $\textbf{Q}_{j \to \overline{f}_{n}}$
is a row block matrix with the only nonzero block $$\left[\textbf J^{\ast}_{j\to  \overline{f}_{n}}\right]^{-1/2}
\overline{\textbf{A}}_{k,j}^T
\textbf{T}_{k,j}^{-1}
\overline{\textbf{A}}_{k,z}
\left[\textbf{J}^{\ast}_{z\to  \overline{f}_{k}}\right]^{-1/2}$$ located at the position corresponding to the position
$\bm\beta_{z\to \overline{f}_k}^{\left(\ell\right)}$ in $\bm\beta^{\left(\ell\right)}$.
Then let $\textbf{Q}$ be a  matrix that stacks  $\textbf{Q}_{j\to \overline{f}_{n}}$
as its row, where  $j$ and $\overline{f}_n$ are in the loop  with $j\in \mathcal B\left(\overline{f}_n\right)$.
Besides, let $\textbf{c}$ be the vector containing the subvector
$\left[\textbf J^{\ast}_{j\to  \overline{f}_{n}}\right]^{1/2}
\textbf{c}_{j\to  \overline{f}_{n}}$ with the same order as $\textbf{Q}_{j\to \overline{f}_{n}}$ in $\textbf{Q}$. We have
\begin{equation}\label{49}
\bm \beta^{\left(\ell\right)} = -\textbf{Q} \bm \beta^{\left(\ell-1\right)} + \textbf{c}.
\end{equation}
Since $\textbf{Q}$ is a square matrix,
$\rho\left(\textbf{Q}\right)\leq
\sqrt{\rho\left(\textbf{Q}\textbf{Q}^T\right)}$ and therefore $\rho\left(\textbf{Q}\textbf{Q}^T\right)<1$ is the  sufficient condition for the convergence of
$\bm \beta^{\left(\ell\right)}$.
We next investigate the elements in $\textbf{Q}\textbf{Q}^T$.

Due to the single loop structure of the graph, every $\bm\beta_{j\to \overline{f}_n}^{\left(\ell\right)} $ in (\ref{47}) would be dependent on a unique $\bm\beta_{z\to \overline{f}_k}^{\left(\ell\right)}$, where $\overline{f}_k \in \mathcal B\left(j\right)\setminus \overline{f}_n$ and $z \in \mathcal \B\left(\overline{f}_k\right) \setminus j$ (i.e., the message two hops backward along the loop in the factor graph).
Thus, the position of the non-zero block in $\textbf{Q}_{j\to \overline{f}_n}$ will be different and non-overlapping for different combinations of ($j,\overline{f}_n$).
As a result,
there exists a column permutation matrix $\bm \Xi$ such that $\textbf{Q}\bm\Xi$ is a block diagonal matrix.
Therefore, $\left(\textbf{Q}\bm\Xi\right)\left(\textbf{Q}\bm\Xi\right)^T
=\textbf{Q}\textbf{Q}^T
$
is also a diagonal matrix, and we can write
\begin{equation}\label{58}
\textbf{Q} \textbf{Q}^{T}=
\texttt{Bdiag}\left\{\textbf{Q}_{j\to \overline{f}_n}\textbf{Q}_{j\to \overline{f}_n}^T
\right\}
_{\textrm{$j\in \mathcal B\left(\overline{f}_n\right)$. } }\nonumber
\end{equation}
As a consequence, $\rho\left(\textbf{Q}\textbf{Q}^T\right)<1$ is equivalent to
$\rho\left(\textbf{Q}_{j\to \overline{f}_n}\textbf{Q}_{j\to \overline{f}_n}^T\right)<1$
for all $j$ and $\overline{f}_n$ in the loop with $j\in \mathcal B\left(\overline{f}_n\right)$.
Following  the definition of $\textbf{Q}_{j\to \overline{f}_n}$ below (\ref{linear-vector}), we obtain
\begin{equation}\label{rowproduct}
\begin{split}
\textbf{Q}_{j\to \overline{f}_n}
\textbf{Q}_{j\to \overline{f}_n}^T
=
&
\left[\textbf J^{\ast}_{j\to  \overline{f}_{n}}\right]^{-1/2}
\overline{\textbf{A}}_{k,j}^T
\textbf{T}_{k,j}^{-1}
\overline{\textbf{A}}_{k,z}
\left[\textbf J^{\ast}_{z\to  \overline{f}_{k}}\right]^{-1}
\overline{\textbf{A}}_{k,z}^T
\textbf{T}_{k,j}^{-1}
\overline{\textbf{A}}_{k,j}
\left[\textbf J^{\ast}_{j\to  \overline{f}_{n}}\right]^{-1/2}
\\
=&
\left[\textbf J^{\ast}_{j\to  \overline{f}_{n}}\right]^{-1/2}
\overline{\textbf A}_{k,j}^T
\textbf{T}_{k,j}^{-1}
\left(\textbf{T}_{k,j}-\overline{\textbf{R}}_{k}
\right)
\textbf{T}_{k,j}^{-1}
\overline{\textbf A}_{k,j}
\left[\textbf J^{\ast}_{j\to  \overline{f}_{n}}\right]^{-1/2},
\end{split}
\end{equation}
where the second equation  follows from the definition of $\textbf{T}_{k,j}$ in (\ref{M}).
Besides, since $\overline{\textbf R}_{k} \succ \textbf 0$, we have $ \textbf{T}_{k,j}-\overline{\textbf{R}}_{k}
\prec
\textbf{T}_{k,j}
$.
Following P B.2 in Appendix B, and due to $\textbf{T}_{k,j}=\textbf{T}_{k,j}^T$, we have
\begin{equation}\label{temp}
\begin{split}
\textbf{T}_{k,j}^{-1/2}
\left(
\textbf{T}_{k,j}-\overline{\textbf{R}}_{k}
\right)
\textbf{T}_{k,j}^{-1/2}
\prec \textbf{I}.
\end{split}
\end{equation}
Applying P B.2 in Appendix B again to (\ref{temp}), and making use of (\ref{rowproduct}), we obtain
\begin{equation}\label{QQineq}
\begin{split}
\textbf{Q}_{j\to \overline{f}_{n}}\textbf{Q}_{j\to \overline{f}_{n}}^T
&\prec
\left[\textbf J^{\ast}_{j\to  \overline{f}_{n}}\right]^{-1/2}
\overline{\textbf A}_{k,j}^T
\textbf{T}_{k,j}^{-1}
\overline{\textbf A}_{k,j}
\left[\textbf J^{\ast}_{j\to  \overline{f}_{n}}\right]^{-1/2}.
\end{split}
\end{equation}

According to (\ref{M}), we have
\begin{equation}\label{amplify}
\overline{\textbf A}_{k,j}^T
\textbf{T}_{k,j}^{-1}
\overline{\textbf A}_{k,j}
=\overline{\textbf A}_{k,j}^T
\left[\overline{\textbf{R}}_k
+
\overline{ \textbf{A}}_{k,z}
\left[\textbf{J}^{\ast}_{z\to \overline{f}_k}\right]^{-1}\overline{\textbf{A}}_{k,z}^T
 \right]^{-1}
\overline{\textbf A}_{k,j}.
\end{equation}
On the other hand, using
(\ref{Cov}),
due to $\mathcal B\left(\overline{f}_k\right)\setminus j=\textbf x_z$ in the considered topology,
the right hand side of (\ref{amplify}) is
$\textbf J^{\ast}_{\overline{f}_{k}\to j} $.
Therefore,
(\ref{QQineq})
 is further written as
 \begin{equation}\label{57}
\begin{split}
\textbf{Q}_{j\to \overline{f}_{n}}\textbf{Q}_{j\to \overline{f}_{n}}^T
&\prec
\left[\textbf J^{\ast}_{j\to  \overline{f}_{n}}\right]^{-1/2}
\textbf J^{\ast}_{\overline{f}_{k}\to j}
\left[\textbf J^{\ast}_{j\to  \overline{f}_{n}}\right]^{-1/2}.
\end{split}
\end{equation}
From (\ref{v2fV}),  $\textbf J^{\ast}_{j\to  \overline{f}_{n}} =
\textbf{W}_j^{-1} +
\textbf J^{\ast}_{\overline{f}_{k}\to j}
+
\sum_{\widetilde{f}_k\in\B\left(j\right)\setminus \overline{f}_n}
\textbf{J}_{\widetilde{f}_k\to j}^{\ast}
$,
thus
$\textbf J^{\ast}_{\overline{f}_{k}\to j}
\preceq\textbf J^{\ast}_{j\to  \overline{f}_{n}} $.
Therefore,
$
\left[\textbf J^{\ast}_{j\to  \overline{f}_{n}}\right]^{-1/2}
\textbf J^{\ast}_{\overline{f}_{k}\to j}
\left[\textbf J^{\ast}_{j\to  \overline{f}_{n}}\right]^{-1/2}
\preceq
\textbf{I}$, and, together with
(\ref{57}), we have
\begin{equation}
\textbf{Q}_{j\to \overline{f}_{n}}\textbf{Q}_{j\to \overline{f}_{n}}^T
\prec
\textbf{I}.\nonumber
\end{equation}
Hence $\rho\left(\textbf{Q}_{j\to \overline{f}_{n}}\textbf{Q}_{j\to \overline{f}_{n}}^T\right)<1$
for all $j$ and $\overline{f}_n$ in the loop and $j\in \mathcal B\left(\overline{f}_n\right)$,
and equivalently $\rho\left(\textbf{Q}\right)<1$.
This completes the proof.
\end{proof}

\subsection{Convergence of Belief Covariance and Mean Vector}\label{C}
As the computation of the belief covariance $\textbf{P}_i^{\left(\ell\right)}$ depends on the message information matrix $
\textbf{J}_{f_n\to i}^{\left(\ell\right)}$,
using
 Theorems \ref{guarantee}
and   \ref{RateCov}, we can derive  the convergence and uniqueness properties of $\textbf{P}_i^{\left(\ell\right)}$.

Before we present the main result, we first present  some properties of the part metric  $\mathrm{d}\left( \textbf{X}, \textbf{Y}\right)$,  with positive definite arguments $ \textbf{X}$, $ \textbf{Y}$, and $\triangle \textbf{X}$.
The proofs are provided in Appendix D.

\begin{proposition} \label{P_Part}
	The part metric  $\mathrm{d}\left( \textbf{X}, \textbf{Y}\right)$  satisfies the following properties
\end{proposition}

\noindent P \ref{P_Part}.1:
$\mathrm{d} \left( \textbf{X}_1+ \textbf{X}_2, \textbf{Y}_1+ \textbf{Y}_2\right)\leq
\mathrm{d} \left(\textbf {X}_1,  \textbf{Y}_1\right)
+
\mathrm{d} \left( \textbf{X}_2, \textbf{Y}_2\right)$;

\noindent P \ref{P_Part}.2:
$\mathrm{d} \left(\textbf {X}, \textbf {Y}\right)=
\mathrm{d} \left( \textbf{X}^{-1}, \textbf{Y}^{-1}\right).
$
\\
\\
\noindent
We now have the following result.

\begin{corollary} \label{C-converge-iff2}
	With arbitrary initial message information matrix $\textbf{J}^{\left(0\right)}_{f_n\to i}  \succeq \textbf{0}$ for all $i\in \mathcal V$ and $f_n \in \mathcal B\left(i\right)$,
	the belief covariance matrix $\textbf{P}_i^{\left(\ell\right)}$ converges to a unique p.d. matrix
	 at a doubly exponential rate  with respect to any matrix norm before $\textbf{P}_i^{\left(\ell\right)}$ enters $\textbf{P}_i^{\ast}$'s neighborhood, which can be chosen to be arbitrarily small.
\end{corollary}
\begin{proof}
{
Since $\textbf{J}_{f_n\to i}^{\left(\ell\right)}$ converges to a p.d. matrix, and according to (\ref{beliefcov}), $\textbf{P}_i^{\left(\ell\right)}$ also converges. Below, we study the convergence rate of $\textbf{P}_i^{\left(\ell\right)}$.
According to the definition of $\textbf P_i^{(\ell)}$   in (\ref{beliefcov}) and part metric in Definition \ref{mydef}, we have
$$
\mathrm{d}
\left([ \textbf{P}_i^{\left(\ell\right)}]^{-1},
[ \textbf{P}_i^{\ast}]^{-1}
\right)
=
\mathrm{d}
\left( \textbf{W}_i^{-1}+
\sum_{f_n\in\B\left(i\right)}
\textbf{J}_{f_n\to i}^{\left(\ell\right)},
\textbf{W}_i^{-1}+\sum_{f_n\in\B\left(i\right)}
\textbf{J}_{f_n\to i}^{\ast}
\right).
$$
By applying  P \ref{P_Part}.1 to the above equation, we obtain
$$
\mathrm{d}
\left([ \textbf{P}_i^{\left(\ell\right)}]^{-1},
[ \textbf{P}_i^{\ast}]^{-1}
\right)
\leq
\mathrm{d}
\left(
\textbf{W}_i^{-1},
\textbf{W}_i^{-1}
\right)
+
\sum_{f_n\in\B\left(i\right)}
\mathrm{d}
\left(
\textbf{J}_{f_n\to i}^{\left(\ell\right)},
\textbf{J}_{f_n\to i}^{\ast}
\right)
=
\sum_{f_n\in\B\left(i\right)}
\mathrm{d}
\left(
\textbf{J}_{f_n\to i}^{\left(\ell\right)},
\textbf{J}_{f_n\to i}^{\ast}
\right).
$$
According to (\ref{rate}), for all $i\in \mathcal V$ and $f_n\in \mathcal B(i)$, there exist a $c<1$ such that
$$
\mathrm{d}
\left(
\textbf{J}_{f_n\to i}^{\left(\ell\right)},
\textbf{J}_{f_n\to i}^{\ast}
\right)
<
c^{\ell}
\mathrm{d}
\left(
\textbf{J}_{f_n\to i}^{\left(0\right)},
\textbf{J}_{f_n\to i}^{\ast}
\right).
$$
Applying the above inequality to compute $\left[\textbf P_i^{(\ell)}\right]^{-1}$ in (\ref{beliefmean}), we obtain
$$
\mathrm{d}
\left([ \textbf{P}_i^{\left(\ell\right)}]^{-1},
[\textbf {P}_i^{\ast}]^{-1}
\right)
<
c^{\ell}\sum_{f_n\in\B\left(i\right)}
\mathrm{d}
\left(
\textbf{J}_{f_n\to i}^{\left(0\right)},
\textbf{J}_{f_n\to i}^{\ast}
\right).
$$
Following P \ref{P_Part}.2, the above inequality is equivalent to
$$
\mathrm{d}
\left( \textbf{P}_i^{\left(\ell\right)},
\textbf{P}_i^{\ast}
\right)
<
c^{\ell}\sum_{f_n\in\B\left(i\right)}
\mathrm{d}
\left(
\textbf{J}_{f_n\to i}^{\left(0\right)} ,
\textbf{J}_{f_n\to i}^{\ast}
\right),
$$
where $\sum_{f_n\in\B\left(i\right)}
\mathrm{d}
\left(
\textbf{J}_{f_n\to i}^{\left(0\right)} ,
\textbf{J}_{f_n\to i}^{\ast}
\right)$ is a constant.
Following the same procedure as that from (\ref{rate}) to (\ref{rateineq2}), we can prove that $ \textbf{P}_i^{\left(\ell\right)}$
converges at a doubly exponential rate  with respect to the monotone norm before $\textbf{P}_i^{\left(\ell\right)}$ enters $\textbf{P}_i^{\ast}$'s neighborhood, which can be chosen to be arbitrarily small.
}
\end{proof}

On the other hand, as shown in (\ref{beliefmean}), the computation of the belief mean $\boldsymbol{\mu}_i^{\left(\ell\right)}$ depends on the belief covariance $\textbf{P}_i^{\left(\ell\right)}$ and the message mean $\textbf{v}^{\left(\ell\right)}_{f_n\to i}$.
Thus, under the same condition as  in Theorem \ref{meanvector}, $\boldsymbol{\mu}^{\left(\ell\right)}_i$ is convergence guaranteed.
Moreover,  it is shown in \citep[Appendix]{MRFtoFG} that, for Gaussian BP over a factor graph,  the converged value of belief mean equals the optimal estimate in (\ref{Central}).
Together with the convergence guaranteed topology
revealed in Theorem \ref{ref}, we have the following Corollary.
\begin{corollary} \label{mean-if2}
With  arbitrary  $\textbf{J}_{f_n\to i}^{\left(0\right)}\succeq \textbf 0$
and arbitrary $\textbf v^{\left(0\right)}_{f_n\to i}$ for all $i\in \mathcal V$ and $f_n \in \mathcal B\left(i\right)$,
the mean vector  $\boldsymbol{\mu}_i^{\left(\ell\right)}$ in  (\ref{beliefmean}) converges
  to the optimal estimate $\widehat{\textbf x}_i$ in (\ref{Central})
 if and only if $\rho\left(\textbf{Q}\right)<1$,
 where $\textbf{Q}$ is defined in (\ref{meanvectorupdate}).  Furthermore, a sufficient condition to guarantee $\rho\left(\textbf{Q}\right)<1$ is when the
factor graph contains only one single loop connected to  multiple chains/trees.
\end{corollary}

{
\section{Relationships with Existing Convergence Conditions}\label{relationship}
In this section,
we show the relationship between our convergence condition for Gaussian BP and the recent proposed path-sum method \citep{pathsum}.
We also show that our convergence condition is more general than the walk-summable condition \citep{WalkSum1} for the scalar case.

\subsection{Relationship with the Path-sum Method}
The path-sum method is proposed in \citep{giscard2012walk, giscard2013evaluating, pathsum} to compute  $
\left(\textbf{W}^{-1}+  \textbf{A}^T\textbf{R}^{-1}\textbf{A}\right)^{-1}$ in (\ref{Central}), in which
the  matrix inverse $\left(\textbf{W}^{-1}+ \textbf{A}^T\textbf{R}^{-1}\textbf{A}\right)^{-1}$ is interpreted as
 the sum of simple paths and simple cycles on a weighted graph.
The resulting formulation is   guaranteed to  converge to the correct value for any valid multivariate Gaussian distribution.

The BP message update equations (\ref{v2fV}), (\ref{v2fm}), (\ref{Cov}), and (\ref{f2vmm})  can be seen as a cut-off of path-sum  by retaining only self-loops and backtracks (simple cycles of lengths one and two).
In the presence of a graph with one or more loops, equations (\ref{v2fV}), (\ref{v2fm}), (\ref{Cov}), and (\ref{f2vmm}) do not include the terms
related to simple cycles with length larger than $2$.
 This may be a potential cause  for the possible divergence of the Gaussian BP algorithm.
From this perspective, the divergence can be averted if none of the walks going around the loop(s) have weight greater than one, or equivalently, that the spectral radius of the block matrix representing the loop(s) is strictly less than one. This is an intuitive explanation of the condition $\rho(\textbf Q) < 1$ obtained in Theorem \ref{meanvector}.  It also immediately follows from these considerations that the convergence rate is at least geometric, with a cut-off of order $\ell$ yielding an $\mathcal O(\rho (\textbf Q)^{\ell})$ error\footnote{We thank an anonymous reviewer for this interpretation.}.

While the path-sum framework provides an insightful interpretation of the results obtained in this paper, the path-sum algorithm may not be efficiently implementable in  distributed and parallel settings, as it requires the summation over all the paths of any length.  In contrast,
Gaussian BP, while paying the price of non-convergence in general loopy
 models, makes it possible to realize  parallel and fully distributed inference.
In summary,  though the path-sum method converges for arbitrary valid Gaussian models,
it is difficult to be adapted to a distributed and parallel inference setup as the Gaussian BP method.

\subsection{Relationship with the Walk-Summable Condition}
We show next that, in the setup of linear Gaussian models, the condition $\rho(\textbf Q)<1$ as in Corollary \ref{mean-if2} encompasses the Gaussian MRF based walk-summable  (\cite{WalkSum1}) in terms of  convergence.
As all existing results on Gaussian BP convergence   \citep{WalkSum1,minsumquad}  only apply to scalar variables, we restrict the following discussion to only the scalar case.
In \citep{WalkSum1},  the starting point for the convergence analysis for Gaussian MRF  is  a joint multivariate Gaussian distribution
\begin{equation}\label{56}
q(\textbf x)\propto\exp\Big\{
-\frac{1}{2} {\textbf x}^T \textbf J\textbf x + \textbf{h}^T \textbf x
\Big\},
\end{equation}
  expressed in the normalized information form  such that $\textbf J_{i,i} =1$ for all $i$.
The underlying Gaussian distribution is  factorized as (\cite{WalkSum1})
\begin{equation}\label{MRF-join}
q\left({\textbf x}\right)
\propto
\prod_{i\in \mathcal{V}} \psi_{i} \left(\textbf  x_i \right)
\prod_{ J_{i,j}\neq 0; i\leq j } \psi_{i,j} \left(\textbf  x_i, \textbf {x}_j\right),
\end{equation}
where
\begin{equation}\label{pairwise}
\psi_{i} \left(\textbf  x_i\right) = \exp\left\{-\frac{1}{2}\textbf J_{i,i} \textbf x^2_i +\textbf h_i\right\}
\quad
\textrm{and} \quad
\psi_{i,j} \left(\textbf x_i, \textbf x_j\right) = \exp\left\{- \textbf x_i \textbf J_{i,j}\textbf  x_j\right\}.\nonumber
\end{equation}

In  \citet[Proposition 1]{WalkSum1},
based on the interpretation that $\left[\textbf J^{-1}\right]_{i,j}$ is the sum of the weights of all the walks from variable $j$ to variable $i$ on the corresponding Gaussian MRF, a sufficient Gaussian BP convergence condition  known as walk-summability  is provided, which is  equivalent to
\begin{equation}\label{walksummability}
\textbf I - |{\textbf R}| \succ \textbf 0,
\end{equation}
together with the initial message variance inverse being set to  $0$, where ${\textbf R} = \textbf I - {\textbf J}$ and $|\textbf R|$ is matrix of entrywise absolute values of $\textbf R$.
In the following,
we  establish the relationship between walk-summable Gaussian MRF and linear Gaussian model by utilizing properties of H-matrices \citep{boman2005factor}.
We show that, with Gaussian MRF satisfying the walk-summable condition, the joint distribution $q(\textbf x)$ in (\ref{MRF-join}) can be reformulated as a special case of the linear Gaussian model based factorization in  (\ref{jointpost}).
Moreover, Gaussian BP  on this particular linear Gaussian model  always converges.
\begin{definition}\label{Hmatrix}
	H-Matrices \citep{boman2005factor}:
A matrix $\textbf X$ is an H-matrix if all eigenvalues of the matrix $\mathcal H(\textbf X)$, where $[\mathcal H(\textbf X)]_{i,i}=| \textbf X_{i,i}|$, and $[ \mathcal H(\textbf X)]_{i,j}
= - |\textbf X_{i,j}|$ have positive real parts.
\end{definition}

\begin{proposition} \label{HmatrixP}
Factor width at most $2$ factorization \citep[Theorem 9]{boman2005factor}:
A symmetric  H-matrix $\textbf X$ that has non-negative diagonals can always be factorized as
$\textbf X = \textbf V\textbf V^T
$,
where $\textbf V  $ is a real  matrix with each column of $\textbf V  $ containing at most $2$ non-zeros.
\end{proposition}

Let
$\omega$ be an arbitrary positive value that is smaller than the minimum eigenvalue of $\textbf I - |{\textbf R}|$ and also satisfies $0<\omega<1$.
According to (\ref{walksummability}),
we have
$(1-\omega)\textbf I - |{\textbf R}|  \succ \textbf 0$.
Furthermore, by applying $\mathcal H(\cdot)$ to $(1-\omega)\textbf I - {\textbf R} $, we have
$[\mathcal H((1-\omega)\textbf I - {\textbf R} )]_{i,i}
= |(1-\omega)\textbf I - {\textbf R} |_{i,i}=
1-\omega$
 and
$[\mathcal H((1-\omega)\textbf I - {\textbf R} )]_{i,j}
=-|[(1-\omega)\textbf I - {\textbf R} )]_{i,j}|
= - |{\textbf R}_{i,j}|$.
Thus,
$\mathcal H((1-\omega)\textbf I - {\textbf R} )
= (1-\omega)\textbf I - |{\textbf R}|\succ \textbf 0$, and
 we conclude that $(1-\omega)\textbf I - {\textbf R} $ is an H-matrix.
According to Proposition \ref{HmatrixP}, $(1-\omega)\textbf I - {\textbf R}
=
{\textbf J} -\omega\textbf I = \textbf V\textbf V^T,
$
where each column of $\textbf V  $ contains at most $2$ non-zeros.
Now, we can rewrite the  joint distribution in (\ref{MRF-join}) as
\begin{equation}
\begin{split}
q({\textbf x})
&\propto
\exp\left\{
-\frac{1}{2}{\textbf x}^T\left( {\textbf J}-\omega\textbf I\right) {\textbf x} -\frac{1}{2}\omega{\textbf x}^T
{\textbf x}+{\textbf h}^T{\textbf x}
\right\}\\
&=\exp\left\{
-\frac{1}{2}
\left( {\textbf V}^T {\textbf x} \right)^T
\left( {\textbf V}^T {\textbf x} \right)
-\frac{1}{2}\left(
\omega{\textbf x}^T
{\textbf x}-2{\textbf h}^T{\textbf x}\right)
\right\}\\
&\propto
\exp\left\{-\frac{1}{2}\sum_{n=1}^M
\left( V_{n,n_i} x_{n_i} + V_{n,n_j}  x_{n_j}\right)^2 -\frac{1}{2} \sum_{n=1}^M \omega( x_n - \frac{ h_n}{\omega})^2\right\},
\end{split}
\end{equation}
where $ V_{n,n_i}$ and $ V_{n,n_j}$ denote the two possible non-zero elements on the $n$-th column and $n_i$-th and $n_j$-th rows, and $M$ is the dimension of $\textbf x$.
Thus, a walk-summable Gaussian MRF in (\ref{56}) (or equivalently (\ref{MRF-join})) can always be written as
 \begin{equation}
q({\textbf x})\propto \prod_{n=1}^M
\underbrace{\mathcal N( x_n|\frac{ 1}{\omega} h_n,\frac{1}{\omega})}_{p
	( x_n)}
\prod_{n=1}^M
\underbrace{\mathcal N(0| V_{n,n_i}  x_{n_i} + V_{n,n_j}  x_{n_j},1)}_{f_n}.
 \end{equation}
Note that, in the above equation, $p( x_n)$ serves as the prior distribution for $ x_n$ as that in (\ref{jointpost}) and  $f_n$ is the local likelihood function with local observation being $y_n= 0$ and noise distribution $ z_n\sim \mathcal N( z_n|0, 1)$\footnote{For a particular $f_n$, if there is only one non-zero coefficient,  $f_n\times \mathcal N( x_n|\frac{ 1}{\omega} h_n,\frac{1}{\omega})$  is also proportional to a Gaussian distribution, which can be seen as a prior distribution in (\ref{jointpost}).}.
Thus the above equation is a special case of the linear Gaussian model based factorization in (\ref{jointpost}) with the local likelihood function $f_n$ containing only a pair of variables.
For this pairwise linear Gaussian model with scalar variables, it is shown in \citep{minsumquad} that $\rho(\textbf Q)<1$  is fulfilled.
Thus by Corollary \ref{mean-if2}, Gaussian BP always converges.
In summary,  for the factorization based on Gaussian MRF, if the walk-summable convergence condition is fulfilled, there is an equivalent joint distribution factorization based on linear Gaussian model; and Gaussian BP   is convergence guaranteed for this linear Gaussian model.

In the following, we further  demonstrate through an example that
 there exist Gaussian MRFs, in which the information matrix $\textbf J$ fails to satisfy the walk-summable condition, but a convergence guaranteed  Gaussian BP update based on the distributed linear Gaussian model representation can still be obtained.
 More specifically,
consider the following  information matrix $\textbf J$ in a Gaussian MRF:
\begin{equation}\label{J}
\textbf J=\left[
\begin{matrix}
\begin{smallmatrix}
1&\frac{1}{3\sqrt{2}}& \frac{1}{\sqrt{3}} & \frac{\sqrt{2}}{3}  \\
\frac{1}{3\sqrt{2}}&1&0&\frac{1}{3}	\\
\frac{1}{\sqrt{3}}&0&1&\frac{1}{\sqrt{6}}\\
\frac{\sqrt{2}}{3}&\frac{1}{3}&\frac{1}{\sqrt{6}}&1
\end{smallmatrix}
\end{matrix}
\right].
\end{equation}
The eigenvalues of $\textbf I - |\textbf R|$ to $4$ decimal places are  $-0.0754$,
$0.9712$, $1.4780$, and $1.6262$.
According to the walk-summable definition in (\ref{walksummability}),  it   is non walk-summable and the convergence condition in \citep{WalkSum1} is inconclusive as to whether Gaussian BP  converges.
On the other hand, we can
study the Gaussian BP convergence of this example by employing a  linear Gaussian model representation, and rewriting
$\textbf J$  as
$\textbf J = \textbf A^T\textbf R^{-1}\textbf A + \textbf W^{-1}$, where
$$
\textbf A=\left[
\begin{matrix}
\frac{2}{\sqrt{6}}& 0 & \frac{1}{\sqrt{2}} & \frac{1}{\sqrt{3}} \\
\frac{{1}}{\sqrt{6}}&\frac{{1}}{\sqrt{3}}&0&0	\\
0&\frac{{1}}{\sqrt{3}}&0&\frac{{1}}{\sqrt{3}}
\end{matrix}
\right],
\quad
\textbf W=\left[
\begin{matrix}
6& 0 & 0 & 0 \\
0& 3&0&0	\\
0&0&2&0\\
0&0&0&3
\end{matrix}
\right],
$$
and $\textbf R = \textbf I$.
In Fig.~\ref{equivalent}, the joint distribution of this example with Gaussian MRF and  the corresponding linear Gaussian model  are represented by factor graphs.
As it is shown in  Corollary  \ref{mean-if2}, for a factor graph
that is the union of a forest and a single loop, as  in  Fig.~\ref{equivalent}(b), Gaussian BP always converges  to the exact value.
This is in sharp contrast to the inconclusive convergence property when
 the same joint distribution is expressed using
 the classical  Gaussian MRF   in (\ref{MRF-eqn}).

In summary, we have shown that the linear Gaussian model with $\rho(\textbf Q)<1$  encompasses walk-summable Gaussian MRF.
Further, it is shown in \citep{WalkSum1} that the diagonally dominant convergence condition in \citep{DiagnalDominant} for Gaussian BP is a special case of the walk-summable condition.
Also, the convergence condition in \citep{Suqinliang} is encompassed by walk-summable condition.
Therefore, we have the Venn diagram in Fig.~\ref{venn} summarizing the relations (in terms of convergence guarantees) between the  convergence condition proposed in this paper and existing conditions.
\begin{figure}[t]
	\centering
	\epsfig{file=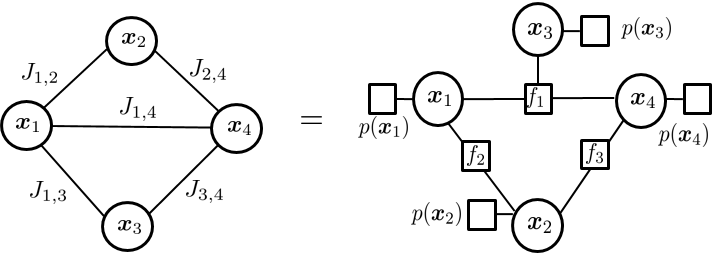, width=3.2in}
	\caption{ The Gaussian MRF corresponding to $\textbf J$ in (\ref{J}) with the factorization following (\ref{MRF-eqn});
		(b) The factor graph  corresponding to $\textbf J$ in (\ref{J}) with the factorization following (\ref{jointpost}).}
	\label{equivalent}
\end{figure}

\begin{figure}[t]
	\centering
	\epsfig{file=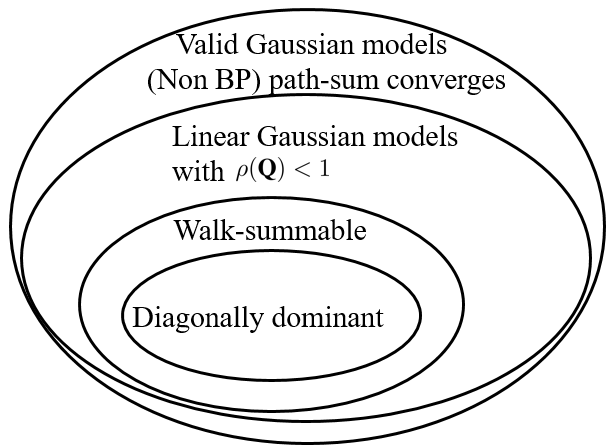, width=2.6in}
	\caption{ Venn diagram summarizing various subclasses of Gaussian models: the three inner most conditions are all for the BP algorithm while the path-sum  in general does not constitute a  BP algorithm.}
	\label{venn}
\end{figure}}

\section{Conclusions}\label{conclusion}
This paper shows that, depending on how the factorization of the underlying joint Gaussian distribution is performed, Gaussian belief propagation (BP)  may exhibit different convergence properties as different factorizations  lead to fundamentally different recursive update structures.
The paper studies the convergence of Gaussian BP  derived from the factorization based on the distributed linear Gaussian model.
We show that the condition we present for convergence of the marginal mean based on
 factorizations using
 the linear Gaussian model is more general than the walk-summable condition \citep{WalkSum1} (and references therein) that is based on the Gaussian Markov random field factorization.
Further, the linear Gaussian model that is studied in this paper
readily conforms to the physical network topology
arising in large-scale networks.

Further, the paper analyzes the
 convergence  of the
Gaussian BP based distributed inference algorithm.
In particular, we show
 analytically that,
with arbitrary positive
semidefinite matrix  initialization,
the  message information matrix exchanged among agents  converges   to a unique positive definite  matrix, and it approaches an arbitrarily small neighborhood of this unique positive definite matrix at a doubly exponential rate (with respect to any matrix norm).
Regarding the initial information matrix,  there exist  positive definite initializations  that guarantee faster convergence
than
 the commonly used all-zero matrix.
Moreover, under the positive
semidefinite initial message information matrix  condition, we present  a necessary and sufficient  condition of the belief mean
vector  to converge to the optimal centralized estimate.
We also prove that Gaussian BP always converges if
the corresponding factor graph is a union of a single loop and a forest.
{In particular, we show
that the proposed convergence condition for Gaussian BP based on the linear Gaussian model leads to a strictly larger class of models in which Gaussian BP converges than those postulated by the Gaussian Markov random field based walk-summable  condition.}
{
Finally, we discuss connections of Gaussian BP with the general path-sum algorithm.
In the future, it will be interesting to explore if these  path-sum interpretations can lead to  modifications of the standard Gaussian BP algorithm that
 guarantee the convergence of Gaussian BP for larger classes of topologies while being also parallel and fully distributed. }



\newpage

\appendix
\section*{Appendix A.}
\label{A}
We first compute the first round updating message from variable node to factor node.
Substituting $m^{\left(0\right)}_{f_n \to i}\left(\textbf{x}_i\right)
\propto
\exp\left\{-\frac{1}{2}||\textbf x_j- \textbf{v}^{\left(0\right)}_{f_n \to i}||^2_{\textbf{J}^{\left(0\right)}_{f_n \to i}}\right\}
$ into (\ref{BPv2f1}) and, after   algebraic manipulations, we obtain
\begin{equation} \label{A-1}
m^{\left(1\right)}_{j \to f_n}\left(\textbf x_j\right)\propto
\exp\left\{-\frac{1}{2}||\textbf x_j- \textbf{v}^{\left(1\right)}_{j\to f_n}||^2_{\textbf{J}^{\left(1\right)}_{j \to f_n}}\right\},
\end{equation}
with
\begin{equation}
\textbf{J}^{\left(1\right)}_{j \to f_n}
= \textbf{W}_j^{-1} +
\sum_{f_k\in\B\left(j\right)\setminus f_n}
\textbf{J}_{f_k\to j}^{\left(0\right)},\nonumber
\end{equation}
and
\begin{equation}
\textbf{v}^{\left(1\right)}_{j\to f_n}=
\left[\textbf{J}^{\left(1\right)}_{j\to f_n}\right]^{-1}
\left[ \sum_{f_k\in\B\left(j\right)\setminus f_n}
\textbf{J}_{f_k\to j}^{\left(0\right)}
\textbf{v}^{\left(0\right)}_{f_k\to j}\right].\nonumber
\end{equation}

Next, we evaluate $m^{\left(1\right)}_{f_n \to i}\left(\textbf{x}_i\right)
$.
By substituting
$m^{\left(1\right)}_{j \to f_n}\left(\textbf x_j\right)$ in (\ref{A-1}) into
(\ref{BPf2v1}), we obtain
\begin{equation}\label{integral}
\begin{split}
m^{\left(1\right)}_{f_n \to i}\left(\textbf{x}_i\right)
&\propto \int\ldots \int
\exp\Big\{-\frac{1}{2}\big(\textbf{y}_n - \sum_{j\in  \mathcal{B}\left(f_n\right)}\textbf{A}_{n,j}\textbf x_j\big)^T\textbf R^{-1}\big(\textbf{y}_n - \sum_{j\in  \mathcal{B}\left(f_n\right)}\textbf{A}_{n,j}\textbf x_j\big)\Big\} \times
\\
&\quad \quad
\!\!\prod_{j\in\B\left(f_n\right)\setminus i}
\exp\Big\{-\frac{1}{2}||\textbf x_j-
\textbf{v}^{ \left(1\right)}_{j\to f_n}||^2_{
\textbf{J}^{ \left(1\right)}_{j\to f_n}}\Big\}\, \mathrm{d}\{\textbf x_j\}_{j\in\B\left(f_n\right)\setminus i}.
\end{split}
\end{equation}

Let $\textbf x_{\left\{\B\left(f_n\right)\setminus i\right\}}$ and
$\textbf{v}_{\left\{\B\left(f_n\right)\setminus i\right\}\to f_n}^{\left(1\right)}$ be stacked vectors containing $\textbf x_j$ and
$\textbf{v}^{\left(1\right)}_{j\to f_n}$ as vector elements for all $j\in \B\left(f_n\right)\setminus i$ arranged in ascending order on $j$, respectively;
$\textbf{A}_{n,\left\{\B\left(f_n\right)\setminus i\right\}}$ denotes a row block matrix containing $\textbf{A}_{n,j}$ as row elements for all $j\in \B\left(f_n\right)\setminus i$ arranged in ascending order;
and $\textbf{J}^{ \left(1\right)}_{\left\{\B\left(f_n\right)\setminus i\right\}\to f_n}$ is a block diagonal matrix with $\textbf{J}^{ \left(1\right)}_{j\to f_n}$ as its block diagonal elements
for all $j\in \B\left(f_n\right)\setminus i$ arranged in ascending order.
Then, (\ref{integral}) can be reformulated as
\begin{eqnarray}\label{integral2}
m^{(1)}_{f_n \to i}(\textbf{x}_i) \nonumber
\!\!\!&\propto&\!\!\!\!\int\ldots \int
\exp\left\{-\frac{1}{2}
||\textbf{y}_n - \textbf{A}_{n,i}\textbf{x}_i - \textbf{A}_{n,\{\B(f_n)\setminus i\}}\textbf x_{\{\B(f_n)\setminus i\}} ^T||^2_{\textbf R^{-1}}
\right\}
\\ \nonumber
&&\quad\times
\exp \left\{-\frac{1}{2}||\textbf x_{\{\B(f_n)\setminus i\}}- \textbf{v}_{\{\B(f_n)\setminus i\}\to f_n}^{(1)})||^2_{\widetilde{\textbf{J}}^{(1)}_{\{\B(f_n)\setminus i\}\to f_n}}
\right\} \mathrm{d} \textbf x_{\{\B(f_n)\setminus i\}}
\\\nonumber
&\propto &
\label{integral-1}
\exp\left\{-\frac{1}{2}||\textbf{y}_n - \textbf{A}_{n,i}\textbf{x}_i||^2_{\textbf R^{-1}} \right\}  \\
&&              \label{intsqur}
\times\int\ldots \int \exp\left\{\!-\!\frac{1}{2}
    (\textbf x_{\{\B(f_n)\setminus i\}}^T\textbf{K}^{(1)}_{f_n\to i} \textbf x_{\{\B(f_n)\setminus i\}}\! \!-\!2[\textbf{h}^{(1)}_{f_n\to i}]^T\textbf x_{\{\B(f_n)\setminus i\}}) \!\! \right\}\!\! \mathrm{d}\textbf x_{\{\B(f_n)\setminus i\}},
\end{eqnarray}
where
\begin{equation}
\textbf{K}^{\left(1\right)}_{f_n\to i}=\textbf{A}_{n,\left\{\B\left(f_n\right)\setminus i\right\}}^T\textbf R^{-1}\textbf{A}_{n,\left\{\B\left(f_n\right)\setminus i\right\}}
+
\textbf{J}^{\left(1\right)}_{\left\{\B\left(f_n\right)\setminus i\right\}\to f_n}\nonumber
\end{equation}
and
\begin{equation}
\textbf{h}^{\left(1\right)}_{f_n\to i}= \textbf{A}_{n,\left\{\B\left(f_n\right)\setminus i\right\}}^T\textbf R^{-1}\left(\textbf{y}_n - \textbf{A}_{n,i}\textbf{x}_i\right)  +\textbf{J}^{\left(1\right)}_{\left\{\B\left(f_n\right)\setminus i\right\}\to f_n}
\textbf{v}_{\left\{\B\left(f_n\right)\setminus i\right\}\to f_n}^{\left(1\right)}.\nonumber
\end{equation}
By completing the square for the integrand of (\ref{intsqur}), we obtain
\begin{eqnarray}\label{integral2}
m^{\left(1\right)}_{f_n \to i}\left(\textbf{x}_i\right)
\!\!\!& \propto \!\!\!&
\alpha_{f_n \to i}^{\left(1\right)}\exp\left\{-\frac{1}{2}
||\textbf{y}_n - \textbf{A}_{n,i}\textbf{x}_i||^2_{\textbf R_n^{-1}}
 + \frac{1}{2}\left[\textbf{h}^{\left(1\right)}_{f_n\to i}\right]^T\left[\textbf{K}^{\left(1\right)}_{f_n\to i}\right]^{-1}\textbf{h}^{\left(\ell\right)}_{f_n\to i} \right\},
\end{eqnarray}
with
\begin{equation}
\alpha_{f_n \to i}^{\left(1\right)} =  \label{integral-2}
\int\ldots \int \exp\left\{-\frac{1}{2}
    ||\textbf x_{\left\{\B\left(f_n\right)\setminus i\right\}} -\left[\textbf{K}^{\left(1\right)}_{f_n\to i}\right]^{-1}
    \textbf{h}^{\left(1\right)}_{f_n\to i} ||^2_{\textbf{K}^{\left(1\right)}_{f_n\to i}}
  \right\} \mathrm{d}\textbf x_{\left\{\B\left(f_n\right)\setminus i\right\}}.\nonumber
\end{equation}

Next, by
applying the spectral theorem to $\textbf{K}^{\left(1\right)}_{f_n\to i}$ and after some
algebraic manipulations, we  simplify (\ref{integral2}) as
\begin{equation}\label{f2vmess2-1}
m^{\left(1\right)}_{f_n \to i}\left(\textbf{x}_i\right)
\propto \alpha_{f_n \to i}^{\left(1\right)}
\exp
\left\{-\frac{1}{2}
||\textbf{x}_i- \textbf{v}^{\left(1\right)}_{f_n\to i}||
_{\textbf{J}^{\left(1\right)}_{f_n\to i}}
\right\},\nonumber
\end{equation}
with the inverse of the covariance, the information matrix
\begin{equation}\label{Cov-new-1}
\begin{split}
\textbf{J}^{\left(1\right)}_{f_n\to i}
& =
\textbf{A}_{n,i}^T
\left[ \textbf{R}_n
+
\sum_{j\in\B\left(f_n\right)\setminus i} \textbf{A}_{n,j}
\left[\textbf{J}^{\left(1\right)}_{j\to f_n}\right]^{-1}\textbf{A}_{n,j}^T \right]^{-1}
\textbf{A}_{n,i},\nonumber
\end{split}
\end{equation}
and the mean vector
\begin{equation}\label{f2vmm-new-1}
\begin{split}
\textbf{v}^{\left(1\right)}_{f_n\to i}
=
\left[\textbf{J}_{f_n\to i}^{\left(1\right)}\right]^{-1}
\textbf{A}_{n,i}^H
\left[ \textbf{R}_n
+
\sum_{j\in\B\left(f_n\right)\setminus i} \textbf{A}_{n,j}
\left[\textbf{J}^{\left(1\right)}_{j\to f_n}\right]^{-1}\textbf{A}_{n,j}^T \right]^{-1}
\left(\textbf{y}_n-\sum_{j\in\B\left(f_n\right)\setminus i} \textbf{A}_{n,j}
\textbf{v}^{\left(1\right)}_{j\to f_n}\right),\nonumber
\end{split}
\end{equation}
and
\begin{equation}\label{eigintegral-2}
\alpha_{f_n \to i}^{\left(1\right)}
\propto
\int\ldots \int \exp\big\{-\frac{1}{2}
    \textbf z^T
    \boldsymbol{\Lambda}_{f_n \to i}^{\left(1\right)}
    \textbf z \big\}    \, \mathrm{d}\textbf z,
\end{equation}
where $\boldsymbol{\Lambda}_{f_n \to i}^{\left(1\right)}$
is a diagonal matrix containing the eigenvalues of
$\textbf{A}_{n,\left\{\B\left(f_n\right)\setminus i\right\}}^T \textbf{R}_n^{-1}\textbf{A}_{n,\left\{\B\left(f_n\right)\setminus i\right\}}
+\textbf{J}^{\left(1\right)}_{\left\{\B\left(f_n\right)\setminus i\right\}\to f_n}$.

By induction, and following similar derivations as in (\ref{A-1}) to (\ref{eigintegral-2}), we obtain the general updating expressions as in (\ref{BPvs2f1}) to (\ref{eigintegral-1}).

\section*{Appendix B.}
\label{B}
Before going into the proof of Lemma \ref{pdlemma}, we note the following properties of positive definite (p.d.) matrices.
If  $\textbf{X} \succ \textbf{0}$, $\textbf{Y} \succ \textbf{0}$, $\textbf{Z} \succeq \textbf{0}$ are of the same dimension, then
we have \citep{JianCFO}:

\noindent P B.1:
$\textbf{X}+\textbf{Y} \succ \textbf{0}$ and
$\textbf{X}+\textbf{Z} \succ \textbf{0}$.

\noindent P B.2:
$\textbf{A}^T\textbf{X}\textbf{A}\succ \textbf{0} $, $\textbf{A}^T\textbf{Z} \textbf{A}\succeq \textbf{0} $, $\textbf{A} \textbf{X} \textbf{A}^T \succeq \textbf{0} $
and $\textbf{A} \textbf{Z} \textbf{A}^T \succeq \textbf{0} $ for any full column rank matrix $\textbf{A}$ with compatible dimension.


Now, we prove Lemma \ref{pdlemma}.
If
$\textbf{J}_{f_k\to j}^{\left(\ell-1\right)}\succeq \textbf{0}$ for all
$f_k\in \B\left(j\right)\setminus f_n$, according to P B.1,
$\sum_{f_k\in\B\left(j\right)\setminus f_n}
\textbf{J}_{f_k\to j}^{\left(\ell-1\right)}\succeq \textbf{0}$.
As $\textbf{W}_j^{-1}\succ \textbf{0}$, we have
$ \textbf{W}_j^{-1} +
\sum_{f_k\in\B\left(j\right)\setminus f_n}
\textbf{J}_{f_k\to j}^{\left(\ell-1\right)}\succ \textbf{0}$, which, according to (\ref{v2fV}), is equivalent to $\textbf{J}^{\left(\ell\right)}_{j \to f_n}\succ \textbf{0}$.
Besides,
as $\textbf{A}_{n,j}$ is full column rank,
if $\left[\textbf{J}^{\left(\ell\right)}_{j\to f_n}\right]^{-1}\succ \textbf{0}$ for all $j\in \B\left(f_n\right)\setminus i$,
according to P B.2, $ \textbf{A}_{n,j}\left[\textbf{J}^{\left(\ell\right)}_{j\to f_n}\right]^{-1}\textbf{A}_{n,j}^T\succeq \textbf{0}$.
With $\textbf{R}_n\succ \textbf{0}$, following P B.1, we have $\left[ \textbf{R}_n
+
\sum_{j\in\B\left(f_n\right)\setminus i} \textbf{A}_{n,j}
\left[\textbf{J}^{\left(\ell\right)}_{j\to f_n}\right]^{-1}\textbf{A}_{n,j}^T \right]^{-1}\succ \textbf{0}$.
As $\textbf{A}_{n,i}$ is of full column rank, by  applying P B.2 again, we have $ \textbf{A}_{n,i}^T
\left[ \textbf{R}_n
+
\sum_{j\in\B\left(f_n\right)\setminus i} \textbf{A}_{n,j}\left[\textbf{J}^{\left(\ell\right)}_{j\to f_n}\right]^{-1}\textbf{A}_{n,j}^T \right]^{-1}
\textbf{A}_{n,i}\succ\textbf{0}$,
which according to  (\ref{Cov})
is equivalent to $\textbf{J}^{\left(\ell\right)}_{f_n\to i} \succ\textbf{0}$.

In summary, we have proved that 1) if $\textbf{J}_{f_k\to j}^{\left(\ell-1\right)}\succeq \textbf{0}$
for all $f_k \in \B\left(j\right)\setminus f_n$, then
$\textbf{J}^{\left(\ell\right)}_{j \to f_n}\succ \textbf{0}$;
2) if
$\left[\textbf{J}^{\left(\ell\right)}_{j\to f_n}\right]^{-1} \succ \textbf{0}$ for all $j\in \B\left(f_n\right) \setminus i$, then $\textbf{J}^{\left(\ell\right)}_{f_n\to i} \succ \textbf{0}$.
Therefore, by setting $\textbf{J}_{f_k\to j}^{\left(0\right)}\succeq \textbf{0}$  for all $k \in \mathcal{V}$ and $j \in \mathcal{B}\left(f_k\right)$,
according to the results of the first case,
we have  $\textbf{J}_{j\to f_n}^{\left(1\right)}\succ \textbf{0}$
for all $j \in \mathcal{V}$ and $f_n \in \mathcal{B}\left(j\right)$.
Then, applying the second case, we further have
$\textbf{J}_{f_n\to i}^{\left(1\right)}\succ \textbf{0}$ for all
$n\in \mathcal V$ and $i\in \mathcal B\left(f_n\right)$.
By repeatedly using the above arguments, it  follows readily that
$\textbf{J}_{f_k \to j}^{\left(\ell\right)} \succ \textbf{0}$ and  $\textbf{J}_{j\to f_n}^{\left(\ell\right)}  \succ \textbf{0}$
for $\ell\geq 1$ and with $j\in \mathcal{V}$, $f_n, f_k\in \mathcal B\left(j\right)$.
Furthermore, according to the discussion before Lemma \ref{pdlemma},  all  messages $m^{\left(\ell\right)}_{j \to f_n}\left(\textbf x_j\right)$ and
$m^{\left(\ell\right)}_{f_n \to i}\left(\textbf{x}_i\right)$ exist, and are in Gaussian form as in (\ref{BPvs2f1}) and (\ref{f2v}).

\section*{Appendix C.}
\label{C}
First, Proposition \ref{P_FUN}, P \ref{P_FUN}.1 is proved.
Suppose that ${\textbf{J}}^{\left(\ell\right)}\succeq
 {\textbf{J}}^{\left(\ell-1\right)}\succeq \textbf{0}$,
i.e.,
$\textbf{J}_{f_k\to j}^{\left(\ell\right)} \succeq
\textbf{J}_{f_k\to j}^{\left(\ell-1\right)}
\succeq \textbf{0}$ for all ${\left(f_k, j\right)\in \mathcal{\widetilde{B}}\left(f_n, i\right)}$,
we have
\begin{equation}
 \textbf{W}_{j}^{-1} +
\sum_{f_k\in\B\left(j\right)\setminus f_n}
\textbf{J}_{f_k\to j}^{\left(\ell\right)}
\succeq
 \textbf{W}_{j}^{-1} +
\sum_{f_k\in\B\left(j\right)\setminus f_n}
\textbf{J}_{f_k\to j}^{\left(\ell-1\right)}\succ \textbf{0}.\nonumber
\end{equation}
Then, according to the fact that if $\textbf{X}\succeq \textbf{Y}\succ \textbf{0}$,
$\textbf{Y}^{-1}\succeq \textbf{X}^{-1}\succ \textbf{0}$, we have
\begin{equation}
\left[
 \textbf{W}_{j}^{-1} +
\sum_{f_k\in\B\left(j\right)\setminus f_n}
\textbf{J}_{f_k\to j}^{\left(\ell-1\right)}
\right]^{-1}
\succeq
\left[
 \textbf{W}_{j}^{-1} +
\sum_{f_k\in\B\left(j\right)\setminus f_n}
\textbf{J}_{f_k\to j}^{\left(\ell\right)}
\right]^{-1}\succ \textbf{0}.\nonumber
\end{equation}
Since $\textbf{A}_{n,j}$ is of full column rank and following P B.2 in Appendix B, we have
$$\textbf{A}_{n,j}
\left[
 \textbf{W}_{j}^{-1} +\!\!\!\!\!
 \sum_{f_k\in\B\left(j\right)\setminus f_n}
\!\textbf{J}_{f_k\to j}^{\left(\ell-1\right)}
\right]^{-1}\!\!\!\!\textbf{A}_{n,j}^T
\succeq
\textbf{A}_{n,j}
\left[
 \textbf{W}_{j}^{-1} +\!\!\!\!\!
 \sum_{f_k\in\B\left(j\right)\setminus f_n}
\!\textbf{J}_{f_k\to j}^{\left(\ell\right)}
\right]^{-1}\!\!\!\!\textbf{A}_{n,j}^T \succ \textbf{0}.$$
Following the same procedure of the proof above and due to  $\textbf R\succ \textbf{0}$, we can further prove that
\begin{equation}
\begin{split}
&\textbf{A}_{n,i}^T
\left[ \textbf{R}_n
+ \sum_{j\in\B\left(f_n\right)\setminus i} \textbf{A}_{n,j}
\left[
 \textbf{W}_{j}^{-1} +
\sum_{f_k\in\B\left(j\right)\setminus f_n}
\textbf{J}_{f_k\to j}^{\left(\ell\right)}
\right]^{-1}
\textbf{A}_{n,j}^T \right]^{-1}
\textbf{A}_{n,i}\\
\succeq &
\textbf{A}_{n,i}^T
\left[ \textbf{R}_n
+ \sum_{j\in\B\left(f_n\right)\setminus i} \textbf{A}_{n,j}
\left[
 \textbf{W}_{j}^{-1} +
\sum_{f_k\in\B\left(j\right)\setminus f_n}
\textbf{J}_{f_k\to j}^{\left(\ell-1\right)}
\right]^{-1}
\textbf{A}_{n,j}^T \right]^{-1}
\textbf{A}_{n,i},\nonumber
\end{split}
\end{equation}
which is equivalent to  $$\mathcal{F}_{n\to i}\left(\left\{
\textbf{J}_{f_k\to j}^{\left(\ell\right)}\right\}_{\left(f_k, j\right)\in \mathcal{\widetilde{B}}\left(f_n, i\right)}  \right) \succeq \mathcal{F}_{n\to i}\left(\left\{
\textbf{J}_{f_k\to j}^{\left(\ell-1\right)}\right\}_{\left(f_k, j\right)\in \mathcal{\widetilde{B}}\left(f_n, i\right)}  \right).$$
Since $\mathcal{F}$ contains $\mathcal{F}_{n\to i}\left(\cdot \right)$ as its component, Proposition \ref{P_FUN}, P \ref{P_FUN}.1 is proved.

Next, Proposition \ref{P_FUN}, P \ref{P_FUN}.2 is proved.
Suppose that
$\textbf{J}_{f_k\to j}^{\left(\ell\right)} \succ \textbf{0}$ for all ${\left(f_k, j\right)\in \mathcal{\widetilde{B}}\left(f_n, i\right)}$.
As $\alpha>1$, we have
\begin{equation} \label{65}
\alpha\textbf{W}_{j}^{-1} +
\sum_{f_k\in\B\left(j\right)\setminus f_n}
\alpha \textbf{J}_{f_k\to j}^{\left(\ell\right)}
\succeq
\textbf{W}_{j}^{-1} +
\sum_{f_k\in\B\left(j\right)\setminus f_n}
\alpha \textbf{J}_{f_k\to j}^{\left(\ell\right)}\succ \textbf{0},\nonumber
\end{equation}
where the equality holds when $\textbf{W}_{j}^{-1}= \textbf{0} $, which corresponds to non-informative prior for $\textbf x_j$.
Applying the fact that
if $\textbf{X}\succeq \textbf{Y}\succ \textbf{0}$,
$\textbf{Y}^{-1}\succeq \textbf{X}^{-1}\succ \textbf{0}$, and, according to  P B.2 in Appendix B, we obtain
\begin{equation}
\textbf{A}_{n,j}
\left[
 \textbf{W}_{j}^{-1} + \!\!\!\!\!\!
\sum_{f_k\in\B\left(j\right)\setminus f_n}\!\!\!\!\alpha
\textbf{J}_{f_k\to j}^{\left(\ell\right)}
\right]^{-1}\!\!\!
\textbf{A}_{n,j}^T
\succeq
\textbf{A}_{n,j}
\left[
\alpha \textbf{W}_{j}^{-1} +  \!\!\!\!\!\!
\sum_{f_k\in\B\left(j\right)\setminus f_n}\!\!\!\! \alpha
\textbf{J}_{f_k\to j}^{\left(\ell\right)}
\right]^{-1}\!\!\!
\textbf{A}_{n,j}^T\succeq \textbf{0}.\nonumber
\end{equation}
Since $\textbf{R}_n\succ\frac{1}{\alpha} \textbf{R}_n \succ \textbf{0}$,
we have
\begin{eqnarray}
&&\left[\frac{1}{\alpha} \textbf{R}_n
+\!\!\! \sum_{j\in\B\left(f_n\right)\setminus i} \textbf{A}_{n,j}
\left[
\alpha \textbf{W}_{j}^{-1} +
\sum_{f_k\in\B\left(j\right)\setminus f_n}\!\!\!\! \alpha
\textbf{J}_{f_k\to j}^{\left(\ell\right)}
\right]^{-1}
\textbf{A}_{n,j}^T \right]^{-1} \nonumber\\
&& \quad \succ
\left[\textbf{R}_n
+\!\!\! \sum_{j\in\B\left(f_n\right)\setminus i} \textbf{A}_{n,j}
\left[
\textbf{W}_{j}^{-1} +
\sum_{f_k\in\B\left(j\right)\setminus f_n} \alpha
\textbf{J}_{f_k\to j}^{\left(\ell\right)}
\right]^{-1}
\textbf{A}_{n,j}^T \right]^{-1}.\nonumber
\end{eqnarray}
Finally, applying P B.2  in Appendix B to the above equation and taking out the common factor $\alpha$, we obtain
\begin{eqnarray}
&&\alpha \textbf{A}_{n,i}^T
\left[ \textbf{R}_n
+ \sum_{j\in\B\left(f_n\right)\setminus i} \textbf{A}_{n,j}
\left[
\textbf{W}_{j}^{-1} +
\sum_{f_k\in\B\left(j\right)\setminus f_n}
\textbf{J}_{f_k\to j}^{\left(\ell\right)}
\right]^{-1}
\textbf{A}_{n,j}^T \right]^{-1}
\textbf{A}_{n,i}\nonumber\\
&&\succ
\textbf{A}_{n,i}^T
\left[ \textbf{R}_n
+ \sum_{j\in\B\left(f_n\right)\setminus i} \textbf{A}_{n,j}
\left[
\textbf{W}_{j}^{-1} +
\sum_{f_k\in\B\left(j\right)\setminus f_n} \alpha
\textbf{J}_{f_k\to j}^{\left(\ell\right)}
\right]^{-1}
\textbf{A}_{n,j}^T \right]^{-1}
\textbf{A}_{n,i}.\nonumber
\end{eqnarray}
Therefore,  $\alpha\mathcal{F}_{n\to i}\bigg(\left\{
\textbf{J}_{f_k\to j}^{\left(\ell\right)}\right\}_{\left(f_k, j\right)\in \mathcal{\widetilde{B}}\left(f_n, i\right)}\bigg) \succ  \mathcal{F}_{n\to i}\left(\left\{
\alpha\textbf{J}_{f_k\to j}^{\left(\ell\right)}\right\}_{\left(f_k, j\right)\in \mathcal{\widetilde{B}}\left(f_n, i\right)}  \right)$
if $\textbf{J}_{f_k\to j}^{\left(\ell\right)}\succ \textbf{0}$
for all ${\left(f_k, j\right)\in \mathcal{\widetilde{B}}\left(f_n, i\right)}$
and $\alpha>1$.
As $\mathcal{F}$ contains $\mathcal{F}_{n\to i}\left(\cdot \right)$ as its component, Proposition \ref{P_FUN}, P \ref{P_FUN}.2 is proved.
In the same way, we can prove
$\mathcal{F}\left(\alpha^{-1}\textbf{J}^{\left(\ell\right)}\right) \succ  \alpha^{-1}\mathcal{F}\left(\textbf{J}^{\left(\ell\right)}\right)$
if $\textbf{J}^{\left(\ell\right)} \succ \textbf{0}$ and $\alpha>1$.

At last, Proposition \ref{P_FUN}, P \ref{P_FUN}.3 is proved.  From Lemma \ref{pdlemma}, if we have initial message information matrix $\textbf{J}^{\left(0\right)}_{f_k \to j} \succeq\textbf{0} $ for all $j \in \mathcal V$ and $f_k \in \mathcal B\left(j\right)$, then we have $\textbf{J}^{\left(\ell\right)}_{f_k \to j} \succ \textbf{0} $ for all $j \in \mathcal V$ and $f_k \in \mathcal B\left(j\right)$.  In such case, obviously, $\textbf{J}^{\left(\ell\right)}\succeq \textbf{0}$.
Applying $\mathcal F$ to both sides of this equation, and using Proposition \ref{P_FUN}, P \ref{P_FUN}.1,
we have $\mathcal{F}\left(\textbf{J}^{\left(\ell\right)}\right) \succeq \mathcal{F}\left(\textbf{0}\right)$.
On the other hand, using (\ref{CovFunc5}), it can be easily checked that $\mathcal{F}\left(\textbf{0}\right)=\textbf{A}^T  \left[  \boldsymbol{\Omega}+ \textbf{H}\boldsymbol{\Psi}^{-1} \textbf{H}^T \right] ^{-1}\textbf{A}\succ \textbf{0}$,
{where the inequality is from Lemma \ref{pdlemma}.}
For proving  the upper bound, we start from the fact that $$\sum_{j\in\B\left(f_n\right)\setminus i} \textbf{A}_{n,j}
\left[
 \textbf{W}_{j}^{-1} +
\sum_{f_k\in\B\left(j\right)\setminus f_n}
\textbf{J}_{f_k\to j}^{\left(\ell-1\right)}
\right]^{-1}
\textbf{A}_{n,j}^T $$ in (\ref{CovFunc}), and equivalently the corresponding term
$$\textbf{H}_{n,i}\left[\textbf{W}_{n,i} + \textbf{K}_{n,i} \left(\textbf{I}_{|\mathcal{B}\left(f_n\right)|-1} \otimes \textbf{J}^{\left(\ell-1\right)}\right) \textbf{K}_{n,i}^T \right]^{-1} \textbf{H}_{n,i}^T$$ in (\ref{CovFunc3}), are  p.s.d. matrices.
In (\ref{CovFunc5}), since $$\textbf{H}\left[\boldsymbol{\Psi} + \textbf{K} \left(\mathbf{I} _{\sum _{n=1} ^M |B\left(f_n\right)|\left(|B\left(f_n\right)|-1\right)} \otimes \textbf{J}^{\left(\ell-1\right)}\right) \textbf{K}^T \right]^{-1} \textbf{H}^T$$ contains
$\textbf{H}_{n,i}\left[\textbf{W}_{n,i} + \textbf{K}_{n,i} \left(\textbf{I}_{|\mathcal{B}\left(f_n\right)|-1} \otimes \textbf{J}^{\left(\ell-1\right)}\right) \textbf{K}_{n,i}^T \right]^{-1} \textbf{H}_{n,i}^T$ as its block diagonal elements, it is also a p.s.d. matrix.
With $ \boldsymbol{\Omega}\succ\textbf{0}$, adding to the above result gives $$ \boldsymbol{\Omega}+\textbf{H}\left[\boldsymbol{\Psi} + \textbf{K} \left(\mathbf{I}_\varphi\otimes \textbf{J}^{\left(\ell\right)}\right) \textbf{K}^T \right]^{-1} \textbf{H}^T \succeq  \boldsymbol{\Omega} \succ\textbf{0}.$$
Inverting both sides, we obtain $  \boldsymbol{\Omega}^{-1} \succeq \left[  \boldsymbol{\Omega}+ \textbf{H}\left[\boldsymbol{\Psi} + \textbf{K} \left(\mathbf{I}_\varphi\otimes \textbf{J}^{\left(\ell\right)}\right) \textbf{K}^T \right]^{-1} \textbf{H}^T \right] ^{-1}.$
Finally, applying P B.2 again gives $$ \textbf{A}^T \boldsymbol{\Omega}^{-1}\textbf{A} \succeq \textbf{A}^T\left[ \boldsymbol{\Omega}+ \textbf{H}\left[\boldsymbol{\Psi} + \textbf{K} \left(\mathbf{I}_\varphi\otimes \textbf{J}^{\left(\ell\right)}\right) \textbf{K}^T \right]^{-1} \textbf{H}^T \right] ^{-1}\textbf{A}^T = \mathcal{F}\left(\textbf{J}^{\left(\ell\right)}\right).$$
Therefore, we have $\textbf{A}^T  \boldsymbol{\Omega}^{-1}\textbf{A}\succeq \mathcal{F}\left(\textbf{J}^{\left(\ell\right)}\right)
\succeq \textbf{A}^T \left[ \boldsymbol{\Omega}+ \textbf{H}\boldsymbol{\Psi}^{-1} \textbf{H}^T \right] ^{-1}\textbf{A}\succ \textbf{0}$.

\section*{Appendix D.}
\label{D}
Let $ \mathrm{d} \left(\textbf {X}_1,  \textbf{Y}_1\right)=\exp\{a_1\}$ and
$\mathrm{d} \left(\textbf {X}_2,  \textbf{Y}_2\right)=\exp\{a_2\}  $,
and
$\mathrm{d} \left(\textbf {X}_1+\textbf {X}_2,  \textbf {Y}_1+\textbf{Y}_2\right)=\exp\{a_3\} $.
First, P \ref{P_Part}.1 is proved.
According to the definition of part metric in Definition \ref{mydef},
for arbitrary symmetric p.d matrix $\textbf X_1$, $\textbf X_2$, $\textbf Y_1$, and $\textbf Y_2$,  we have
$\mathrm{d} \left(\textbf {X}_1,  \textbf{Y}_1\right)$,
$\mathrm{d} \left( \textbf{X}_2, \textbf{Y}_2\right)$, and
$\mathrm{d} \left( \textbf{X}_1+ \textbf{X}_2, \textbf{Y}_1+ \textbf{Y}_2\right)
$ correspond to
\begin{equation}\label{D1}
a_1\textbf X_1
\succeq
\textbf Y_1
\succeq
\frac{1}{a_1}
\textbf X_1,	
\quad
a_2\textbf X_2
\succeq
\textbf Y_2
\succeq
\frac{1}{a_2}
\textbf X_2,	
\end{equation}
\begin{equation}
\label{D2}
a_3(\textbf X_1 + \textbf X_2)
\succeq
\textbf Y_1 + \textbf Y_2
\succeq
\frac{1}{a_3}
\left(\textbf X_1+\textbf X_2\right).	
\end{equation}
Since $ \mathrm{d} \left(\textbf {X}_1,  \textbf{Y}_1\right)>0$ and $ \mathrm{d} \left(\textbf {X}_2,  \textbf{Y}_2\right)>0$, we have $a_1, a_2\geq 1$.
And therefore
$a_1+a_2> a_1$ and $a_1+a_2> a_2$.
 Then, according to (\ref{D1}), we have
\begin{equation}
\label{D3}
(a_1+a_2)\left(\textbf X_1+\textbf X_2\right)\succeq \textbf Y_1 +  \textbf Y_2
\succeq \frac{1}{a_1+a_2}(\textbf X_1+\textbf X_2).
\end{equation}
Following the definition of part matric,  $a_3$ is the smallest value satisfy the inequality in (\ref{D2}).
Thus, by comparing (\ref{D3}) with (\ref{D2}), we obtain
$a_1+a_2\geq a_3$
Hence,
$\mathrm{d} \left( \textbf{X}_1+ \textbf{X}_2, \textbf{Y}_1+ \textbf{Y}_2\right)\leq
\mathrm{d} \left(\textbf {X}_1,  \textbf{Y}_1\right)
+
\mathrm{d} \left( \textbf{X}_2, \textbf{Y}_2\right)$.

Next, P \ref{P_Part}.2 is proved. Following the part metric definition of $\mathrm{d} \left(\textbf {X}_1,  \textbf{Y}_1\right)$,
$a_1\textbf X_1
\succeq
\textbf Y_1
\succeq
\frac{1}{a_1}
\textbf X_1$, which is equivalent to
$
\textbf Y_1^{-1}\succeq \frac{1}{a_1}\textbf X_1^{-1}
$ and
$
{a_1}\textbf X_1^{-1}
\succeq
\textbf Y_1^{-1}
$.
Thus, $\mathrm{d} \left(\textbf {X}, \textbf {Y}\right)=
\mathrm{d} \left( \textbf{X}^{-1}, \textbf{Y}^{-1}\right).
$

\end{document}